\documentclass[reprint,onecolumn,superscriptaddress,amsmath,amssymb,aps,pre,longbibliography,nofootinbib]{revtex4-2}
\usepackage{hyperref}
\usepackage{url}
\usepackage[table, dvipsnames]{xcolor}
\renewcommand{\paragraph}[1]{\noindent \textbf{#1}\quad}
\hypersetup{colorlinks=true,allcolors=NavyBlue}
\usepackage{setspace}
\usepackage[parfill]{parskip}

\usepackage{verbatim}
\usepackage{graphicx}
\usepackage{svg}
\usepackage{caption}
\usepackage{subcaption}

\usepackage{amsmath}
\usepackage{amssymb}
\usepackage{latexsym}
\usepackage{verbatim}
\usepackage{graphicx}
\usepackage{mathtools, nccmath}
\usepackage{cleveref}
\usepackage{amsthm}
\usepackage{subcaption}
\usepackage{cleveref}

\newtheorem{proposition}{Proposition}
\newtheorem{result}{Result}
\newtheorem{lemma}{Lemma}
\newtheorem{corollary}{Corollary}
\newtheorem{remark}{Remark}

\newtheorem{definition}{Definition}

{\begin{list}               
    {$\bullet$ \hfill}{
        \setlength{\leftmargin}{\parindent}
        \setlength{\parsep}{0.04\baselineskip}
        \setlength{\itemsep}{0.5\parsep}
        \setlength{\labelwidth}{\leftmargin}
        \setlength{\labelsep}{0em}}
    }
{\end{list}}

\providecommand{\cref}[1]{Chapter~\ref{chap:#1}}
\providecommand{\sref}[1]{Section~\ref{sec:#1}}

\providecommand{\R}{\ensuremath{\mathbb{R}}}

\providecommand{\N}{\ensuremath{\mathbb{N}}}

\providecommand{\abs}[1]{\lvert#1\rvert}

\providecommand{\norm}[1]{\lVert#1\rVert}

\providecommand{\inprod}[1]{\langle#1\rangle}
\providecommand{\set}[1]{\left\{#1\right\}}

\providecommand{\bydef}{\coloneqq}

\providecommand{\mA}{A}

\providecommand{\w}{\omega}

\newcommand*\diff{\mathop{}\!\mathrm{d}}

\usepackage{commath}
\usepackage{mathtools}
\usepackage{bm}

\newcommand{\E}{\mathbb{E}}
\renewcommand{\P}{\mathbb{P}}
\newcommand{\Ea}[1]{\E\left[#1\right]}
\newcommand{\Eb}[2]{\E_{#1}\left[#2\right]}
\newcommand{\pM}{\Pi}
\newcommand{\ps}{\pi}
\newcommand{\tvs}{\Omega}

\newcommand{\btr}{b_\mathrm{tr}}
\newcommand{\Rtr}{R_\mathrm{tr}}
\newcommand{\Atr}{A_\mathrm{tr}}
\newcommand{\Ate}{A_\mathrm{test}}
\newcommand{\bte}{b_\mathrm{test}}
\newcommand{\Bte}{B_\mathrm{test}}
\newcommand{\Rte}{R_\mathrm{test}}
\newcommand{\Etr}{E_\mathrm{tr}}

\newcommand{\GammaEq}{\Gamma^\ast_\mathrm{eq}}

\DeclareMathOperator{\tr}{tr}
\DeclareMathOperator{\vecop}{vec}

\newcommand{\cl}{\ell} 
\newcommand{\nsamp}{n} 
\newcommand{\load}{\tau} 
\newcommand{\cload}{\alpha} 
\newcommand{\tload}{\kappa} 
\newcommand{\tv}{w}  
\newcommand{\ntv}{k} 
\newcommand{\nv}{\rho} 
 
\newcommand{\nrv}{\epsilon} 
\newcommand{\Ptest}{\mathcal{P}_\mathrm{test}}
\newcommand{\Ptrain}{\mathcal{P}_\mathrm{train}}

\newcommand{\eicl}{e^\mathrm{ICL}}
\newcommand{\eiclrl}{e^\mathrm{ICL}_\mathrm{ridgeless}}
\newcommand{\eidg}{e^\mathrm{IDG}}
\newcommand{\eidgrl}{e^\mathrm{IDG}_\mathrm{ridgeless}}

\newcommand{\unif}[1]{\mathsf{Unif}(#1)}
\providecommand{\sph}{\mathcal{S}^{d-1}(\sqrt{d})}
\providecommand{\unifsp}{\unif{\sph}}
\usepackage{amsmath}
\RequirePackage{etex}
\usepackage{amssymb,amsfonts}
\usepackage{graphicx}
\usepackage[utf8]{inputenc} 
\usepackage[T1]{fontenc}   
\usepackage{url}
\usepackage{enumitem}
\usepackage{autonum}

\makeatletter
\renewcommand{\@fnsymbol}[1]{%
  \ifcase#1\or *\or **\or †\or ‡\or §\else \@arabic{#1}\fi}
\makeatother

\begin{document}

\title{Asymptotic theory of in-context learning by linear attention}

\author{Yue M. Lu}
\email{yuelu@seas.harvard.edu}
\affiliation{The John A. Paulson School of Engineering and Applied Sciences, Harvard University, Cambridge, MA 02138}

\author{Mary I. Letey}
\email{maryletey@fas.harvard.edu}
\affiliation{The John A. Paulson School of Engineering and Applied Sciences, Harvard University, Cambridge, MA 02138}

\author{Jacob A. Zavatone-Veth}
\email[]{jzavatoneveth@fas.harvard.edu}
\altaffiliation[]{Equal contribution.}
\affiliation{The John A. Paulson School of Engineering and Applied Sciences, Harvard University, Cambridge, MA 02138}
\affiliation{Center for Brain Science, Harvard University, Cambridge, MA 02138}
\affiliation{Society of Fellows, Harvard University, Cambridge, MA 02138}
\affiliation{Department of Physics, Harvard University, Cambridge, MA 02138}

\author{Anindita Maiti}
\email[]{amaiti@perimeterinstitute.ca}
\altaffiliation[]{Equal contribution.}
\affiliation{Perimeter Institute for Theoretical Physics, Waterloo, ON N2L 2Y5, Canada}

\author{Cengiz Pehlevan}
\email{cpehlevan@seas.harvard.edu}
\affiliation{The John A. Paulson School of Engineering and Applied Sciences, Harvard University, Cambridge, MA 02138}
\affiliation{Center for Brain Science, Harvard University, Cambridge, MA 02138}
\affiliation{The Kempner Institute for the Study of Natural and Artificial Intelligence, 150 Western Avenue, Alston, MA 02134}

\begin{abstract}
\noindent Transformers have a remarkable ability to learn and execute tasks based on examples provided within the input itself, without explicit prior training. It has been argued that this capability, known as in-context learning (ICL), is a cornerstone of Transformers' success, yet questions about the necessary sample complexity, pretraining task diversity, and context length for successful ICL remain unresolved. Here, we provide a precise answer to these questions in an exactly solvable model of ICL of a linear regression task by linear attention. We derive sharp asymptotics for the learning curve in a phenomenologically-rich scaling regime where the token dimension is taken to infinity; the context length and pretraining task diversity scale proportionally with the token dimension; and the number of pretraining examples scales quadratically. We demonstrate a double-descent learning curve with increasing pretraining examples, and uncover a phase transition in the model's behavior between low and high task diversity regimes: In the low diversity regime, the model tends toward memorization of training tasks, whereas in the high diversity regime, it achieves genuine in-context learning and generalization beyond the scope of pretrained tasks. These theoretical insights are empirically validated through experiments with both linear attention and full nonlinear Transformer architectures. 
\end{abstract}

\maketitle

\setlength{\abovedisplayskip}{6pt}
\setlength{\belowdisplayskip}{6pt}
\setlength{\abovedisplayshortskip}{0pt}
\setlength{\belowdisplayshortskip}{4pt}

\section{Introduction}
Since their introduction by Vaswani et al. in 2017 \cite{vaswani2017attention}, Transformers have become a cornerstone of modern artificial intelligence (AI). Originally designed for sequence modeling tasks, such as language modeling and machine translation, Transformers achieve state-of-the art performance across many domains, even those that are not inherently sequential \cite{dosovitskiy2021image}. Most strikingly, they underpin the breakthroughs achieved by large language models such as BERT \cite{devlin2018bert},
LLaMA \cite{touvron2023llama}, the GPT series \cite{radford2018improving,radford2019language,brown2020language,achiam2023gpt}, the Claude model family \cite{TheC3}, and DeepSeek R1 \cite{deepseekai2025}.

The technological advancements enabled by Transformers have inspired a substantial body of research aimed at understanding their working principles. One key observation is that language models gain new behaviors and skills as their number of parameters and the size of their training datasets grow \cite{brown2020language,ganguli2022predictability,srivastava2022beyond,wei2022emergent}.  A particularly important emergent skill is \textit{in-context learning} (ICL), which describes the model's ability to learn and execute tasks based on the context provided within the input itself, without the need for explicit prior training on those specific tasks. To give an example from natural language processing, a pretrained large language model might be able to successfully translate English to Italian after being prompted with a few example translations, even if it has not been specifically pretrained on that translation task \cite{brown2020language}. ICL enables language models to perform new, specialized tasks without retraining, which is arguably a key reason for their general-purpose abilities. 

Despite recent progress in understanding ICL, fundamental questions about when and how ICL emerges in large language models remain unresolved. These models are typically trained or pretrained using a next-token prediction objective, but the impact of various algorithmic and hyperparameter choices during pretraining on ICL performance is still not well understood. What mechanisms do Transformers implement to facilitate ICL? How many pretraining examples are necessary for ICL capabilities to arise? Furthermore, how many examples must be provided within the input context for the model to successfully perform an in-context task? Another important question is the degree of task diversity in the training data: how diverse must the training tasks be to enable ICL for truly novel tasks that the model has not encountered before?

In this paper, we address these questions by investigating the ICL capabilities of a linear attention module for linear regression tasks. This simplified model setting, which exhibits the minimal architectural features to perform ICL for linear regression, enables us to derive an asymptotically precise theory of ICL performance, elucidating its exact dependence on various hyperparameters. In the remainder of this section, we first provide an overview of related work on ICL. We then summarize our main contributions. 

\subsection{Related Work}

\paragraph{ICL in Transformer architectures.} 
The striking ICL abilities of Transformers were thrust to the fore by Brown \emph{et al.} \cite{brown2020language}'s work on GPT-3. Focusing on natural language processing (NLP) tasks, they showed that ICL performance dramatically improves with an increase in the number of model parameters, with an increase in the number of examples in the model's context, and with the addition of a natural language task description. In subsequent work, Wei \emph{et al.} \cite{wei2022emergent} proposed that the emergence of ICL with increasing scale is an abrupt, unpredictable transition. This perspective has substantially influenced proposed theories for the emergence of ICL \cite{olsson2022context}. However, Schaeffer \emph{et al.} \cite{schaeffer2023mirage} have challenged the idea that the emergence of ICL is unpredictable; they suggest that appropriately-chosen measures of otherwise hidden progress \cite{barak2022hidden} reveal that ICL gradually develops with scale.

\paragraph{Empirical studies of synthetic ICL tasks.} 
Though ICL in NLP is both impressive and useful, these natural data do not allow precise experimental control and study. Towards a fine-grained understanding of the conditions required for ICL, many recent works have explored ICL of parametrically-controllable synthetic tasks, notably linear regression and classification. These works have identified various features of pretraining data distributions that contribute to the emergence of ICL \cite{chan2022data,singh2023transient,bietti2023birth,raventos2023pretraining,reddy2023mechanistic}. Closely related to our work is a study of ICL of linear regression by Ravent\'{o}s \emph{et al.} \cite{raventos2023pretraining}. Their work identified a task diversity threshold for the emergence of ICL, below which a pretrained Transformer behaves as a Bayesian estimator with a prior determined by the limited set of pretraining tasks. Above this threshold, the model's performance approaches that of within-context optimal Bayesian ridge regression, corresponding to a Gaussian prior over all tasks, including those not seen during pretraining. A motivating objective of our work is to provide a theoretical account of the empirical findings made by Ravent\'{o}s \emph{et al.} \cite{raventos2023pretraining}, which underscore the roles of task diversity, regularization, model capacity, and data structure in the emergence of ICL.

\paragraph{Theoretical studies of ICL.} Many theoretical studies of ICL have centered on the idea that Transformers learn a particular algorithm during pretraining, which is then flexibly deployed to solve in-context tasks. In broad strokes, papers from this program of research often consider a particular algorithm for solving an in-context task, prove that Transformers can approximately implement this algorithm, and then empirically compare the ICL performance of a pre-trained Transformer to the performance of that algorithm \cite{bai2023transformers,li2023transformers,akyurek2023what,ahn2023transformers,fu2023transformers,vonoswald2023transformers,zhang2024incontext,zhang2023trained}. A clear consensus on which algorithm underlies ICL of linear regression in full transformers has yet to emerge \cite{bai2023transformers,li2023transformers,akyurek2023what,ahn2023transformers,fu2023transformers,vonoswald2023transformers,zhang2024incontext,zhang2023trained}. Within this line of research, closest to our work are a series of papers that consider ICL of linear regression by simplified Transformers using linear, rather than softmax, attention modules \cite{zhang2023trained,zhang2024incontext,ahn2023transformers,wu2023pretraining,vonoswald2023transformers,chandra2024towards,duraisamy2024finite}. Zhang, Frei, and Bartlett \cite{zhang2023trained} studied these models in the limit of infinite pretraining dataset size (\emph{i.e.}, the population risk limit), and show that their performance on in-context linear regression nearly matches that of the Bayes-optimal estimator for the ICL task. However, they found that linear Transformers are not robust to shifts in the within-context covariate distribution. Zhang, Wu, and Bartlett \cite{zhang2024incontext} then showed that any optimizer of the within-context risk for a linear Transformer solves the ICL task with an approximation to one step of gradient descent from a learnable initialization, and that the resulting estimator can saturate the Bayes error for tasks with a Gaussian prior and non-zero mean. As we will discuss in \Cref{sec:formulation}, our reduction of the linear attention module is inspired in part by these works. In very recent work, Duraisamy \cite{duraisamy2024finite} has studied the fine-sample risk of in-context linear regression with a single step of gradient descent, without directly analyzing Transformers. An \textit{et al.} \cite{ahn2023transformers} and Wu \textit{et al.} \cite{wu2023pretraining} investigated how linear Transformers adapt to limited pretraining data and context length, again showing that in certain cases nearly-optimal error is achievable. Like these studies, our work considers linear attention; but our analysis, with its asymptotically sharp predictions, allows us to elucidate the exact dependence of ICL performance on various hyperparameters and to pinpoint when and how the transition from memorization to ICL of linear regression occurs. In closing, we highlight the work of Reddy \cite{reddy2023mechanistic}, who analyzed the emergence of in-context classification through a phenomenological model.

\subsection{Summary of contributions}

We now summarize the main contributions of our paper relative to the prior art reviewed above. Building on recent literature, we focus on a simplified model of a Transformer that captures its key architectural motif: the linear self-attention module \cite{vonoswald2023transformers, zhang2024incontext, chandra2024towards, wu2023pretraining, zhang2023trained, ahn2023transformers}. Linear attention includes the quadratic pairwise interactions between inputs that lie at the heart of softmax attention, but it omits the normalization steps and fully connected layers. This simplification makes the model more amenable to theoretical analysis. Our main result is a sharp asymptotic analysis of ICL for linear regression using linear attention, leading to a more precisely predictive theory than previous population risk analyses or finite-sample bounds \cite{zhang2023trained, zhang2024incontext}. The main contributions of our paper are structured as follows:

\begin{enumerate}[leftmargin=*]
    \item We begin in \S\ref{sec:formulation} by developing a simplified parameterization of linear self-attention that allows pretraining on the ICL linear regression task to be performed using ridge regression. 

    \item Within this simplified model, we identify a phenomenologically rich scaling limit in which the ICL performance can be analyzed (\S\ref{sec:theory_results}). As the token dimension tends to infinity, we allow the number of pretraining examples scale quadratically with the token dimension, while the context length and pretraining task diversity scale linearly. In this joint limit, we compute sharp asymptotics for ICL performance using random matrix theory. 

    \item The asymptotically precise theory curves we derive reveal several interesting phenomena (\S\ref{sec:theory_results}). First, we observe double-descent in the model's ICL generalization performance as a function of pretraining dataset size, reflecting our assumption that it is pretrained to interpolation. Second, we study the non-monotonic dependence of ICL performance on context length. Last, we uncover a transition from memorization to in-context learning as the pretraining task diversity increases. 
    This transition recapitulates the empirical findings of \cite{raventos2023pretraining} in full Transformer models. 

    \item In \S\ref{sec:experiments}, we demonstrate through numerical experiments that the insights from our theory, derived using the simplified linear attention model, transfer to full Transformer models with softmax self-attention. In particular, the scaling of pretraining sample complexity and task diversity with token dimension required for successful ICL is consistent.
    
\end{enumerate}

\section{Problem formulation}\label{sec:formulation}

We begin by describing the setting of our study.

\subsection{ICL of linear regression} In an ICL task, the model takes as input a sequence of tokens $\lbrace x_1, y_1, x_2, y_2, \ldots, x_{\cl},y_{\cl}, x_{\cl+1}\rbrace$, and outputs a prediction of $y_{\cl+1}$. We will often refer to an input sequence as a ``context.''
The pairs $\lbrace x_i,y_i\rbrace_{i=1}^{\cl+1}$ are i.i.d. samples from a \emph{context-dependent} joint distribution $P(x,y)$. Hence, the model needs to gather information about $P(x,y)$ from the first $\cl$ examples and use this information to predict $y_{\cl+1}$ from $x_{\cl+1}$. We will refer to $\cl$ as the ``context length''.

In this work, we focus on an approximately linear mapping between  $x_i\in \R^d$ and $y_i\in \R$:
\begin{equation}\label{eq:linear_function}
    y_i = \inprod{x_i, \tv} + \nrv_i,
\end{equation}
where $\nrv_i$ is a Gaussian noise $\tv \in \R^d$ is referred to as a task vector. We note that the task vector $\tv$ is fixed within a context, but can change between different contexts. The model has to learn $\tv$ from the $\cl$ pairs presented within the context, and use it to predict $y_{\cl+1}$ from $x_{\cl+1}$.

\subsection{Linear self-attention} 
The model that we will analytically study is the linear self-attention block \cite{wang2020linformer}. Linear self-attention takes as input an embedding matrix $Z$, whose columns hold the sequence tokens. The mapping of sequences to matrices is not unique. Here, following the convention in \cite{zhang2023trained,wu2023pretraining,wang2020linformer}, we will embed the input sequence $\lbrace x_1, y_1, x_2, y_2, \ldots, x_{\cl},y_{\cl}, x_{\cl+1}\rbrace$ as:
\begin{align}\label{eq:Zstructure}
Z = \left[\begin{array}{ccccc} x_1 & x_2 & \ldots & x_{\cl} & x_{\cl+1} \\ y_1 & y_2 & \ldots & y_{\cl} & 0 \end{array}\right] \in \R^{(d+1)\times(\cl+1)},
\end{align}
where $0$ in the lower-right corner is a token that prompts the missing value $y_{\cl+1}$ to be predicted.

For value matrix $V\in\mathbb{R}^{(d+1)\times (d+1)}$ and key and query matrices $K,Q$ such that $K^\top Q\in \mathbb{R}^{(d+1)\times (d+1)}$, the output of a linear-attention block \cite{shen2021efficient,katharopoulos2020transformers,wang2020linformer} is given by
\begin{align}\label{eq:LA}
A \bydef Z + \frac 1{\cl} VZ(KZ)^\top(QZ). 
\end{align} 
The output $A$ is a matrix while our goal is to predict a scalar, $y_{\cl+1}$. Following the choice of positional encoding in \eqref{eq:Zstructure}, we will take $A_{d+1,\cl+1}$, the element of $A$ corresponding to the $0$ prompt, as the prediction for $y_{\cl+1}$:
\begin{align}\label{eq:full_model}
\hat y\bydef A_{d+1,\cl+1}.
\end{align}

\subsection{Pretraining data} The model is pretrained on $n$ sample sequences, where the $\mu$th sample is a collection of $\cl+1$ vector-scalar pairs $\{x_i^{\mu}\in \R^d, y_i^{\mu}\in\R\}_{i=1}^{\cl+1}$ related by the approximate linear mapping in \eqref{eq:linear_function}: $ y_i^{\mu} = \inprod{x_i^{\mu}, \tv^{\mu}} + \epsilon_i^{\mu}$. Here, $\tv^\mu$ denotes the task vector associated with the $\mu$th sample. We make the following statistical assumptions:
\begin{enumerate}[leftmargin=*]
    \item $x_i^{\mu}$ are $d$-dimensional random vectors, sampled i.i.d. over both $i$ and $\mu$ from an isotropic Gaussian distribution $\mathcal{N}(0, I_d/d)$.

    \item At the beginning of training, construct a finite set with $\ntv$ task vectors, denoted by 
    \begin{equation}
        \tvs_\ntv = \set{\tv_1, \tv_2, \ldots, \tv_k}.
    \end{equation}
    The elements of this set are independently drawn once from 
    \begin{equation}\label{eq:tv_training}
        \tv_i \sim_\text{i.i.d.} \mathcal{N}(0,I_d).
    \end{equation}
For $1 \le \mu \le n$, the task vector $\tv^{\mu}$ associated with the $\mu$th sample context is uniformly sampled from $\tvs_\ntv$. Note that the variable $\ntv$ controls the task diversity in the pretraining dataset. Importantly, $\ntv$ can be less than $n$, in which case the same task vector from $\tvs_\ntv$ will be repeated multiple times. 
    
\item The noise terms $\epsilon_i^{\mu}$ are i.i.d. over both $i$ and $\mu$, and drawn  from a normal distribution $\mathcal{N}(0,\nv)$.
\end{enumerate} We denote a sample from this distribution by $(Z,y_{\cl+1}) \sim \Ptrain$.

\subsection{Parameter reduction}\label{sec:reduction}

Before specifying the training procedure, it is insightful to first examine the prediction mechanism of the linear attention module for the ICL task. This proves to be a fruitful exercise, shedding light on key questions: Can linear self-attention learn linear regression in-context? If so, what do the model parameters learn from the data to solve this ICL problem? By closely analyzing these aspects, we can also formulate a simplified problem that lends itself to analytical study.

We start by rewriting the output of the linear attention module, \eqref{eq:full_model}, in an alternative form. Following \cite{zhang2023trained}, we define
\begin{align}\label{eq:VM}
V = \left[\begin{array}{cc} V_{11} & v_{12} \\ v_{21}^{\top} & v_{22}
\end{array}\right], \quad M  = \left[\begin{array}{cc} M_{11} & m_{12} \\ m_{21}^{\top} & m_{22}
\end{array}\right] \bydef K^{\top}Q,
\end{align}
where $V_{11} \in \R^{d \times d}$, $v_{12}, v_{21} \in \R^{d}$, $v_{22}\in \R$, $M_{11} \in \R^{d \times d}$, $m_{12}, m_{21} \in \R^{d}$, and $m_{22}\in \R$. We assume that the inner dimension of $K^{\top}Q$ is greater than equal to $d+1$ so that matrix $M$ can achieve full rank.  From \eqref{eq:LA} and \eqref{eq:full_model}, one can check that
\begin{align} 
\hat y &=  \frac 1{\cl} \big\langle x_{\cl+1}, v_{22}M_{11}^{\top} \sum_{i = 1}^{\cl} y_i x_i+ v_{22}m_{21}\sum_{i=1}^{\cl}y_i^2 + M_{11}^{\top}\sum_{i=1}^{\cl+1} x_{i}x_{i}^{\top}v_{21}+m_{21}\sum_{i=1}^{\cl}y_i x_{i}^{\top}v_{21} \big\rangle,
\end{align}
where $\inprod{\cdot,\cdot}$ stands for the standard inner product.

This expression reveals several interesting points, including how this model could express a solution to the ICL task. First, not all parameters in \eqref{eq:VM} contribute to the output: we can discard all the parameters except for the last row of $V$ and the first $d$ columns of $M$.
Second, the first term 
\begin{equation}
\frac{1}{\cl} v_{22}M_{11}^{\top} \sum_{i = 1}^{\cl} y_i x_i    
\end{equation}
offers a hint about how the linear attention module might be solving the task. The sum $\frac 1{\cl}\sum_{i\leq \cl }y_ix_i$ is a noisy estimate of $\mathbb{E}[xx^{\top}]\tv$ for that context. Hence, if the parameters of the model are such that $v_{22}M_{11}^{\top}$ is approximately $\mathbb{E}[xx^{\top}]^{-1}$, this term alone makes a good prediction for the output. Third, the third term, $ M_{11}^{\top}\sum_{i=1}^{\cl+1} x_{i}x_{i}^{\top}v_{21}$ does not depend on outputs $y$, and thus does not directly contribute to the ICL task that relies on the relationship between $x$ and $y$. Finally, the fourth term, $m_{21}\sum_{i=1}^{\cl}y_i x_{i}^{\top}v_{21}$, only considers a one dimensional projection of $x$ onto $v_{21}$. Because the task vectors $w$ and $x$ are isotropic in the statistical models that we consider, there are no special directions in the problem. Consequently, we expect the optimal $v_{21}$ to be approximately zero by symmetry considerations.

Motivated by these observations, and for analytical tractability, we study the linear attention module with the constraint $v_{21}=0$. In this case, collecting the remaining parameters in a matrix 
\begin{align}\label{eq:Gamma}
    \Gamma \bydef v_{22}\begin{bmatrix} M_{11}^{\top}/d & m_{21}\end{bmatrix} \in \mathbb{R}^{d\times(d+1)}
\end{align}
and the input sequence in another matrix $H_Z$, defined as
\begin{align}\label{eq:H_Z}
    H_Z \bydef  x_{\cl+1} \begin{bmatrix} \frac{d}{\cl} \sum_{i\leq \cl} y_i x_i^{\top} & \frac{1}{\cl}\sum_{i\leq \cl}y_i^2\end{bmatrix} \in \mathbb{R}^{d\times (d+1)},
\end{align}
we can rewrite the predicted label as 
\begin{equation}\label{eq:red}
\hat y = \inprod{\Gamma, H_Z}.
\end{equation}
The $1/d$ scaling of $M_{11}$ in $\Gamma$ is chosen so that the columns of $H_Z$ scale similarly; it does not affect the final predictor $\hat y$. 

We note that \cite{zhang2023trained} provides an analysis of the population risk (whereas we focus on empirical risk) for a related reduced model in which both $v_{21}$ and $m_{21}$ are set to $0$. Consequently, the predictors considered in \cite{zhang2023trained} slightly differ from ours in \eqref{eq:Gamma}--\eqref{eq:red} by an additive term (due to $m_{21}$). The authors of \cite{zhang2023trained} justify this reduced model through an optimization argument: if these parameters $v_{21}$ and $m_{21}$ are initialized to zero, they remain zero under gradient descent optimization of the population risk. 

In the remainder of this paper, we will examine the ICL performance of the reduced model given in \eqref{eq:H_Z} and \eqref{eq:red}, except when making comparisons to a full, nonlinear Transformer architecture. Henceforth, unless explicitly stated otherwise, we will refer to this reduced model as the linear attention module. 

\subsection{Model pretraining} The parameters of the linear attention module are learned from $\nsamp$ samples of input sequences, \begin{align}
\lbrace x_1^{\mu}, y_1^{\mu}, \ldots, x_{\cl+1}^{\mu}, y_{\cl+1}^{\mu}\rbrace, \qquad \mu =1, \ldots, \nsamp.
\end{align}
We estimate model parameters by minimizing MSE loss on next-output prediction with ridge regularistion, giving 
\begin{align}\label{eq:ridge_LT}
    \Gamma^\ast &= \underset{\Gamma}{\arg\,\min}\, \sum_{\mu =1 }^{\nsamp} \left(y_{\cl+1}^{\mu} - \inprod{\Gamma, H_{Z^{\mu}}} \right)^2 +   \frac{\nsamp}{d}\lambda \norm{\Gamma}_\mathrm{F}^2, 
\end{align}
where $\lambda > 0$ is a regularization parameter, and $H_{Z^\mu}$ refers to the input matrix \eqref{eq:H_Z} populated with the $\mu$th sample sequence. The factor $\nsamp/d$ in front of $\lambda$ makes sure that, when we take the $d \to \infty$ or $\nsamp \to \infty$ limits later, there is still a meaningful ridge regularization when $\lambda > 0$. The solution to the optimization problem in \eqref{eq:ridge_LT} can be expressed explicitly as \begin{align}\label{eq:gammastarexplicit}
\text{vec}(\Gamma^\ast) = \left(\frac{\nsamp}{d}\lambda I + \sum_{\mu=1}^{\nsamp}\text{vec}(H_{Z^{\mu}})\text{vec}(H_{Z^{\mu}})^\top\right)^{-1} \sum_{\mu=1}^{\nsamp}y_{\cl+1}^{\mu}\text{vec}(H_{Z^{\mu}}),
\end{align}
where $\vecop(\cdot)$ denotes the vectorization operation. Throughout this paper, we adopt the \emph{row-major} convention. Thus, for a $d_1 \times d_2$ matrix $A$, $\vecop(A)$ is a vector in $\R^{d_1 d_2}$, formed by stacking the rows of $A$ together.

\subsection{Evaluation}\label{sec:evaluation}  For a given set of parameters $\Gamma$, the model's generalization error is defined as
\begin{equation}
    e(\Gamma) \bydef \Eb{\Ptest}{\left(y_{\cl+1} - \inprod{\Gamma, H_Z}\right)^2},
\end{equation}
where $(Z,y_{\cl+1}) \sim \Ptest$ is a new sample drawn from the probability distribution of the test dataset. We consider two different test data distributions $\Ptest$: 
\begin{enumerate}[leftmargin=*]
    \item \emph{ICL task:} $x_i$ and $\epsilon_i$ are i.i.d. Gaussians as in the pretraining case. However, each task vector $\tv^\text{test}$ associated with a test input sequence of length $\ell$ is drawn independently from $\mathcal{N}(0,I_d)$. We will denote the test error under this setting by $\eicl(\Gamma)$.
    
\item \emph{In-distribution generalization (IDG) task:} The test data are generated in exactly the same manner as the training data, \emph{i.e.}, $\Ptest = \Ptrain$, hence the term in-distribution generalization. In particular, the set of unique task vectors $\lbrace\tv_1,\ldots,\tv_{\ntv}\rbrace$ is identical to that used in the pretraining data. We will denote the test error under this setting by $e^{\text{IDG}}(\Gamma)$. This task can also be referred to as in-weight learning. 
\end{enumerate}

The ICL task evaluates the true in-context learning performance of the linear attention module. The task vectors in the test set differ from those seen in training, requiring the model to infer them from context. The IDG task assesses the model's performance on task vectors it has previously encountered during pretraining. High performance on the IDG task but low performance on the ICL task indicates that the model memorizes the training task vectors. Conversely, high performance on the ICL task suggests that the model can learn genuinely new task vectors from the provided context.

To understand the performance of our model on both ICL and IDG tasks, we will need to evaluate these expressions for the pretrained attention matrix $\Gamma^\ast$ given in \eqref{eq:gammastarexplicit}. An asymptotically precise prediction of $\eicl(\Gamma^\ast)$ and $\eidg(\Gamma^\ast)$ will be a main result of this work.

\subsection{Baselines: Bayes-optimal within-context estimators} Following \cite{raventos2023pretraining}, it is useful to compare the predictions made by the trained linear attention to optimal estimators that use only the current context information. These estimators rely solely on the data within the given context for their predictions, but need oracle knowledge of the full statistical models underlying the data. Under the mean square loss, the optimal Bayesian estimator $\hat y_{\text{Bayes}} = \mathbb{E}_{\Ptest}[y_{\cl+1}|x_1, y_1, x_2, y_2, \ldots, x_{\cl},y_{\cl},x_{\cl+1}]$ in our setting has the form
\begin{align}
\hat y_{\text{Bayes}} =(\tv_{\text{Bayes}})^{\top} x_{\cl+1},\end{align}
where $\tv_{\text{Bayes}}$ is the Bayes estimator of the task vector $\tv$.

For the ICL task, the  Bayes-optimal ridge regression estimator is given by 
\begin{align}\label{eq:ridgeestimator}
\tv_{\text{Bayes}}^{\text{ridge}} \bydef \left(\sum_{i=1}^{\cl}x_ix_i^{\top}+\rho I_d\right)^{-1}\left(\sum_{i=1}^{\cl}y_ix_i\right),
\end{align}
where the ridge parameter is set to the noise variance $\rho$. We will refer to it as the \emph{ridge estimator}. 
For the IDG task, the Bayes-optimal estimator is given by
\begin{align}\label{eq:dmmse}
\tv_{\text{Bayes}}^{\text{dMMSE}} \bydef \frac{\sum_{j=1}^{\ntv} w_j e^{-\frac{1}{2\nv}\sum_{i=1}^{\cl}\left(y_i-w_{j}^{\top}x_i\right)^2}}{\sum_{j=1}^{\ntv}  e^{-\frac{1}{2\nv}\sum_{i=1}^{\cl}\left(y_i-w_{j}^{\top}x_i\right)^2}}.
\end{align}
Here, we assume that the training task vectors $\lbrace\tv_1,\ldots,\tv_{\ntv}\rbrace$ are known to the estimator. Following \cite{raventos2023pretraining}, we will refer to this estimator as the \emph{discrete minimum mean squared error (dMMSE) estimator}.

The test performance of these estimators are calculated by 
\begin{equation}
    e^{\text{Bayes}}_{\Ptest} = \Eb{\Ptest}{\left(y_{\cl+1} - (\tv_{\text{Bayes}})^{\top} x_{\cl+1}\right)^2},
\end{equation}
where $\Ptest$ can be associated with either the ICL or the IDG task, and $\tv_{\text{Bayes}}$ can be the ridge or the dMMSE estimator. To avoid possible confusion, we emphasize that we will sometimes plot the performance of an estimator on a task for which it is not optimal. For example, we will test the dMMSE estimator, which is Bayes-optimal for the pretraining distribution, on the ICL task, where it is not optimal. This will be done for benchmarking purposes.

\section{Theoretical results}\label{sec:theory_results}

To answer the questions raised in the introduction, we provide a precise asymptotic analysis of the learning curves of the linear attention module for ICL of linear regression. We then focus on various implications of these equations, and verify through simulations that our insights gained from this theoretical analysis extend to more realistic nonlinear Transformers.

\subsection{Joint asymptotic limit} We have now defined both the structure of the training data as well as the parameters to be optimized. For our theoretical analysis, we consider a joint asymptotic limit in which the input dimension $d$, the pretraining dataset size $\nsamp$, the context length $\cl$, and the number of task vectors in the training set $\ntv$, go to infinity together such that 
\begin{align}\label{eq:scalings}
\frac{\cl}{d} \bydef \cload = \Theta(1), \quad \frac{\ntv}{d} \bydef \tload =  \Theta(1), \quad \frac{\nsamp}{d^2} \bydef \load = \Theta(1).
\end{align}
Identification of these scalings constitutes one of the main results of our paper. 
As we will see, the linear attention module exhibits rich learning phenomena in this limit. 

The intuition for these scaling parameters can be seen as follows. Standard results in linear regression \cite{marchenko1967distribution,bai2010spectral,hastie2022surprises} show that to estimate a $d$-dimensional task vector $w$ from the $\cl$ samples within a context, one needs at least $\cl = \Theta(d)$. The number of unique task vectors that must be seen to estimate the covariance matrix of the true $d$-dimensional task distribution $\mathcal{N}(0,I_d)$ should also scale with $d$, \emph{i.e.} $k = \Theta(d)$.
Finally, we see from \eqref{eq:Gamma} that the number of linear attention parameters to be learned is $\Theta(d^2)$. This suggests that the number of individual contexts the model sees during pretraining should scale similarly, \emph{i.e.}, $n = \Theta(d^2)$.

\subsection{Learning curves for ICL of linear regression by a linear attention module} Our theoretical analysis, explained in detail in the Supplementary Information, leads to asymptotically precise expressions for the generalization errors under the two test distributions being studied. Specifically, our theory predicts that, as $d, \nsamp, \cl, \ntv \to \infty$ in the joint limit given in \eqref{eq:scalings},
\begin{align}
    \eicl(\Gamma^\ast) \longrightarrow \eicl(\tau, \cload, \tload, \nv, \lambda) \qquad \text{almost surely},
\end{align}
and
\begin{align}
    \eidg(\Gamma^\ast) \longrightarrow \eidg(\tau, \cload, \tload, \nv, \lambda) \qquad \text{almost surely},
\end{align}
where $\eicl(\tau, \cload, \tload, \nv, \lambda)$ and $\eidg(\tau, \cload, \tload, \nv, \lambda)$ are two deterministic functions of the parameters $\tau$, $\cload$, $\tload$, $\nv$ and $\lambda$. The exact expressions of these two functions can be found in Section \ref{sec:AsymptoticLimits} of the SI. For simplicity, we only present in what follows the ridgeless limit (\emph{i.e.}, $\lambda \to 0^+$) of the asymptotic generalization errors. \\

\begin{result}[ICL generalization error in the ridgeless limit]\label{res:eg_icl_ridgeless_main} Let
\begin{equation}\label{eq:mu_x_m_main}
    q^\ast \bydef \frac{1+\rho}{\alpha},\quad m^\ast \bydef \mathcal{M}_\kappa\left({q^\ast}\right), \quad \mu^\ast \bydef q^\ast \mathcal{M}_{\tload/\load}(q^\ast),
\end{equation}
where $\mathcal{M}_\kappa(\cdot)$, defined in Section \ref{appendix:Wishart} of the supplementary, is a function related to the Stieltjes transform of the Marchenko-Pastur law. Then
\begin{equation}\label{eq:ICLridgeless}
\begin{aligned}
    \eiclrl &\bydef \lim_{\lambda \to 0^+} \eicl(\load, \cload,\tload, \nv, \lambda) \\
    &= \begin{cases}
    {\frac{\tau(1+q^\ast)}{1-\tau}\left[1-\tau(1-\mu^\ast)^2+\mu^\ast(\rho/q^\ast-1)\right]} {-2\tau(1-\mu^\ast)+(1+\rho)} & \tau < 1\\
    \left(q^\ast+1\right)\left(1 - 2q^\ast m^\ast -(q^\ast)^2 \mathcal{M}'_\tload(q^\ast) + \frac{(\rho + q^\ast  - (q^\ast)^2 m^\ast) m^\ast}{\tau-1}\right) -2 (1-q^\ast m^\ast) + (1+\nv) & \tau > 1
    \end{cases},
\end{aligned}
\end{equation}
where $\mathcal{M}'_\tload(\cdot)$ means the derivative of $\mathcal{M}_\tload(q)$ with respect to $q$. \\
\end{result} 

\begin{result}[IDG generalization error in the ridgeless limit]\label{res:eg_idg_ridgeless_main} Let $q^\ast$, $m^\ast$, and $\mu^\ast$ be the scalars defined in \eqref{eq:mu_x_m_main}. We have
\begin{align}
    &\eidgrl \bydef \lim_{\lambda \to 0^+} \eidg(\load, \cload,\tload, \nv, \lambda) = \\
    &\begin{cases}
    \frac{\tau}{1-\tau}\left(\frac{\rho + q^\ast - 2 q^\ast(1-\tau)({q^\ast}/{\xi^\ast}+1)}{1 - p^\ast(1-\tau)} + \frac{\tau \mu^\ast(q^\ast+\xi^\ast)^2}{q^\ast}\right) & \tau < 1\\
         \frac{\tau}{\tau-1}[\rho + q^\ast (1 - q^\ast m^\ast)] & \tau > 1
    \end{cases},\label{eq:IDGridgeless}
\end{align}
where $\xi^\ast = \frac{(1-\tau)q^\ast}{\tau \mu^\ast}$ and $p^\ast = \big(1 - {\kappa}\big(\frac{\kappa \xi^\ast}{1-\tau}+1\big)^{-2}\big)^{-1}$.
\end{result}

We will discuss various implications of these equations in the next sections.

We derived these results using techniques from random matrix theory. The full setup and technical details are presented in the Supplementary Information. The computations involve analysis of the properties of the finite-sample optimal parameter matrix $\Gamma^*$ (see \eqref{eq:gammastarexplicit}). A key technical component of our analysis involves characterizing the spectral properties of the sample covariance matrix of  $n = \Theta(d^2)$ i.i.d. random vectors in dimension $\Theta(d^2)$. Each of these vectors is constructed as the vectorized version of the matrix in \eqref{eq:H_Z}. Related but simpler versions of this type of random matrices involving the tensor product of i.i.d. random vectors have been studied in recent work \cite{dubova2023universality}. Some of our derivations are based on non-rigorous yet technically plausible heuristics. We support these predictions with numerical simulations and discuss in the Supplementary Information the steps required to achieve a fully rigorous proof. 

\subsection{Sample-wise double-descent}
\begin{figure*}[ht!]
  \centering
  \includegraphics[width=\textwidth]{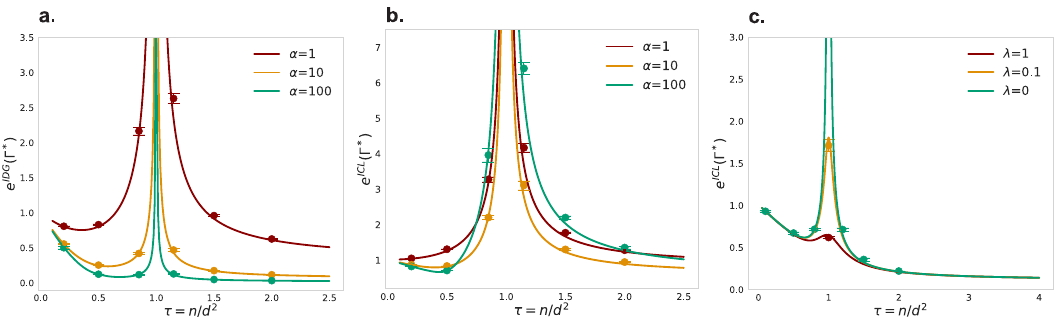}
  \caption{ICL performance as a function of $\load$: theory (solid lines) vs simulations (dots). Plots of  (a.) $\eidgrl(\tau,\alpha,\kappa,\rho)$, (b.) $\eiclrl(\tau,\alpha,\kappa,\rho)$, and (c.) $\eicl(\tau,\alpha,\kappa,\rho,\lambda)$ against $\tau$. Simulated errors calculated by evaluating the corresponding test error on the corresponding optimized $\Gamma^*$. \textit{Parameters:} $d=100$, $\rho = 0.01$ for all; (a.), (b.) $\kappa=0.5$, (c.) $\alpha = 10, \kappa = \infty$. Averages and standard deviations are computed over 10 runs.}
  \label{fig:wide_figure} 
\end{figure*}

How large should $\nsamp$, the pretraining dataset size, be for the linear attention to succesfully learn the task in-context? In \Cref{fig:wide_figure}, we plot our theoretical predictions for the ICL and IDG error as a function of $\tau = \nsamp/d^2$ and verify them with numerical simulations. Our results demonstrate that the quadratic scaling of sample size with input dimensions is indeed an appropriate regime where nontrivial learning phenomena can be observed.

As apparent in \Cref{fig:wide_figure}, we find that the generalization error for both ICL and IDG tasks are not monotonic in the number of samples. In the ridgeless limit, both ICL and IDG errors diverge at $\tau = 1$, with the leading order behavior in the $\tau \uparrow 1$ (respectively $\tau \downarrow 1$) limit given by $\frac {c_1} {1-\tau}$ (respectively $\frac {c_2} {\tau-1}$), where $c_1$ (respectively $c_2$) is a $\tau$-independent constant.  This leads to a ``double-descent'' behavior \cite{belkin2019reconciling,hastie2022surprises} in the number of samples. As in other models exhibiting double-descent \cite{belkin2019reconciling,hastie2022surprises,atanasov2024scaling}, the location of the divergence is at the interpolation threshold: the number of parameters of the model (elements of $\Gamma$) is, to leading order in $d$, equal to $d^2$, which matches the number of pretraining samples at $\load = 1$. Further, we can investigate the effect of ridge regularization on the steepness of the double descent, as illustrated in  \Cref{fig:wide_figure}(c.) for the ICL task. As we would expect from other models exhibiting double-descent \cite{belkin2019reconciling,hastie2022surprises,atanasov2024scaling}, increasing the regularization strength suppresses the peak in error around the interpolation threshold. 

Finally, we note that if we take the limit of $\kappa \to \infty$ and $\alpha \to \infty$ in \Cref{res:eg_icl_ridgeless_main} (in either order), the ICL generalization error in the ridgeless case reduces to the generalization error of simple ridgeless interpolation with isotropic Gaussian covariates in $d^2$ dimensions \cite{hastie2022surprises,atanasov2024scaling}:
\begin{align}\label{eq:linear_reg_limit}
    \lim_{\alpha\to\infty} \lim_{\kappa \to \infty} \eiclrl = \lim_{\kappa \to \infty} \lim_{\alpha\to\infty} \eiclrl = 
    \begin{dcases}
        1 - \tau + \frac{\rho}{1-\tau} & \tau < 1, \\ 
        \frac{\rho\tau}{\tau-1} & \tau>1 .
    \end{dcases}
\end{align}
This result makes sense, given that in this limit the ICL generalization problem reduces to the generalization error of ridge regression in $d^2$ dimensions with covariates formed as the tensor product of i.i.d. Gaussian vectors, which by universality results in \cite{dubova2023universality} should in turn be asymptotically equal to that for isotropic Gaussian covariates \cite{hastie2022surprises}. We note that taking the limit $\alpha \to \infty$ is not strictly necessary for this universality, but doing so simplifies the expressions and makes the double-descent obvious.

\subsection{ICL and IDG error curves can have non-monotonic dependence on context length} 
\begin{figure*}[ht!]
  \centering
  \includegraphics[width=\textwidth]{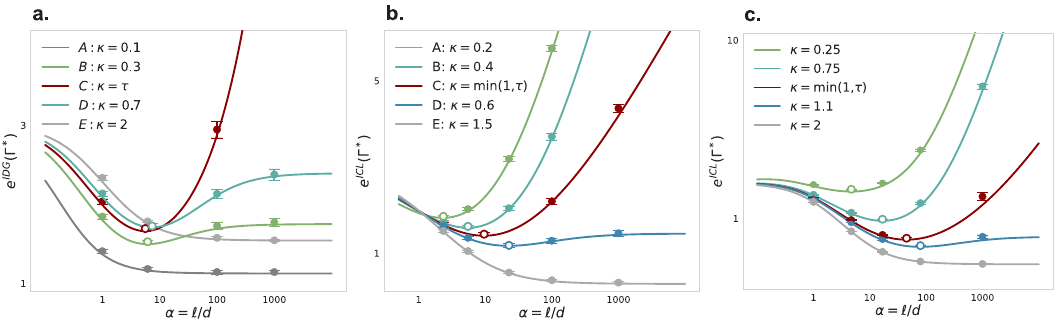}
  \caption{Error curves as functions of $\cload$: theory (solid lines) vs simulations (dots). Plots of {(a.)} $\eidgrl$ and {(b.)}, (c.) $\eiclrl$ against $\alpha$.  Hollowed markers, when plotted, indicate $\alpha^\ast$ value minimizing error if it exists. \textit{Parameters:} $d = 100$; (a.) $\tau=0.5,\rho=0.5$; (b.) $\tau=0.5,\rho=0.1$; (c.) $\tau = 20, \rho=0.5$. Averages and standard deviations are computed over 20-100 runs.}
  \label{fig:linear_alpha}
\end{figure*}

  %

How large should the context length be? In \Cref{fig:linear_alpha}, we plot our theoretical results verified with experiments. We observe that we have correctly identified the regime where in-context and in-weight learning appear: context length indeed scales linearly with input dimensions, as numerical simulations computed using finite $d$ fit the asymptotic error curves.  

An interesting observation is that the IDG and ICL errors do not always monotonically decrease with context length; we see that there are parameter configurations for which the error curves are minimized at some finite $\alpha^\ast$. 

\begin{figure*}[ht!]
  \centering
  \includegraphics[width=\textwidth]{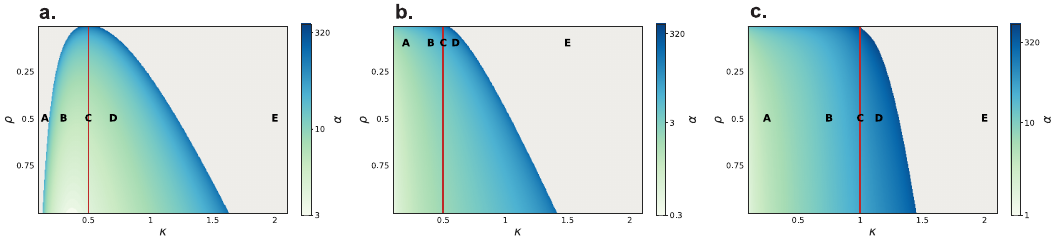}
  \caption{Phase diagram of non-monotonicity in $\alpha$ of ridgeless IDG and ICL error, for fixed $\tau$ against task diversity $\kappa$ and label noise $\rho$. The colormap corresponds to the value $\alpha^\ast$ that minimizes error in $\alpha$, if it exists, at that particular $\kappa, \rho$ pair; grey is plotted if the error curve is monotonic at that $\kappa, \rho$ pair. Figures (a.), (b.), (c.) in this plot correspond to the respective setup in Figure \ref{fig:linear_alpha}, and points A-E correspond to the respective curve in Figure \ref{fig:linear_alpha}. The dashed vertical lines are plotted at $\kappa = \text{min}(\tau,1)$. We know from \eqref{eq:lambdaalphalimit_tsmall} and surrounding discussion that ICL error diverges in $\alpha$ for all $\kappa,\rho$ to the left of this line, and IDG error diverges \textit{on} this line for $\tau < 1$.}
  \label{fig:nonmonotonicity}
\end{figure*}

While the functional form of $\eidgrl$ and $\eiclrl$ is too complex to study their minimizers analytically, we can investigate the non-monotonicity of IDG and ICL error in $\alpha$ numerically. This is done in  \Cref{fig:nonmonotonicity}, where the blue-green colormap corresponds to the value $\alpha^\ast$, if it exists, that minimizes $\eidgrl(\tau,\alpha,\kappa,\rho)$ or $\eiclrl(\tau,\alpha,\kappa,\rho)$ at fixed $\tau,\kappa,\rho$. The grey color indicates cases where the corresponding function is monotonic in $\alpha$. The results reveal that the non-monotonicity of the errors with respect to context length $\alpha$ depends in a nontrivial manner on both the task diversity $\kappa$ and the noise-to-signal ratio $\rho$.

In addition to the non-monotonicity of error curves in $\alpha$, we also notice a divergence in error as $\alpha$ increases for some $\kappa, \tau$ values. To determine when this occurs, we compute the large $\alpha$ limit of the above error curves. For ICL error, we have
\begin{align}\label{eq:lambdaalphalimit_tsmall}
    \lim_{\alpha \to \infty} \eiclrl &= 
    \begin{dcases}
        \infty & \kappa \leq \text{min}(\tau, 1), 
        \\
        1 - \tau + \rho + \frac{\rho \kappa \tau}{(\kappa-\tau)(1-\tau)} & \tau < 1, \kappa > \tau, \\
        \rho + \frac{\rho \kappa}{(\kappa - 1)(\tau-1)} & \tau > 1, \kappa > 1.
    \end{dcases}
\end{align} 
Note the explicit divergence of ICL error for $\kappa \leq \text{min}(\tau, 1)$ (independent of $\rho$). Similar investigation of Result 2 under the limit $\alpha \to \infty$ reveals an explicit divergence of IDG error for $\kappa = \tau$. The divergence in $\alpha$ of ICL error is a peculiarity of the ridgeless limit of errors we have taken in Results 1 and 2, and is further discussed in Section \ref{sec:AsymptoticLimits} of the Supplementary Information.

\subsection{Learning transition with increasing pretraining task diversity} 

\begin{figure*}[ht!]
  \centering
  \includegraphics{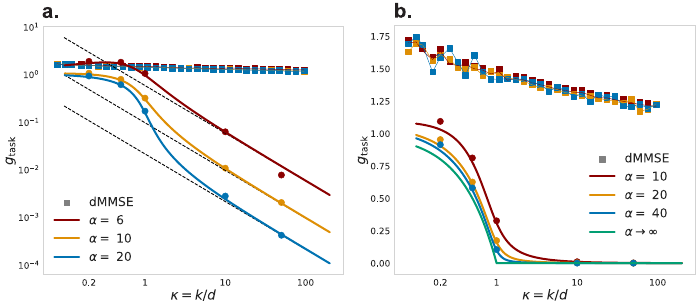}
  \caption{Plots of $g_\text{task}$ against $\kappa$ for linear transformer model vs dMMSE estimator. (a.) has plots on log-log scale highlighting the substantial difference in their rates of convergence towards 0. Dashed and solid lines are theory predictions; dots and squares are numerical simulation. (b.) has plots on log-linear scale with the addition of the green $\alpha\to\infty$ curve given by \eqref{eq:proportionallimit} demonstrating the phase transition at $\kappa = 1$. \textit{Parameters:} Linear transformer: $d = 100 - 140$, $\tau = 0.2\alpha$ throughout, simulations computed over 20 runs; dMMSE: $d=100$, simulations computed over 1000-5000 runs}
  \label{fig:gtask_linear}
\end{figure*}

 It is important to quantify if and when a given model is actually learning in-context, \textit{i.e.}, solving a new regression problem by adapting to the specific structure of the task rather than relying solely on memorized training task vectors. We refer to this phenomenon as \textit{task generalization}. 

We expect pretraining task diversity ($k$) to play a key role in task generalization. To provide some motivation for this, suppose the model was pretrained on only a single task vector ($k=1$). It is reasonable to expect that the model will memorize this task vector and make predictions under the prior that the task vector in the test examples is identical to that particular pretraining task vector. On the other extreme, suppose the model was pretrained by drawing a fresh task vector for every training sequence. Here, the expectation would be that the model achieves task generalization. Indeed, results from previous computational studies on the linear regression task \cite{raventos2023pretraining} suggest that transformers are better in-context learners if they have seen a greater task diversity in the training set, and further the models transition from a task memorization phase to a task generalization phase as task diversity increases. Hence, we study the task generalization behavior of the linear transformer as the task diversity increases and precisely quantify the task diversity scale at which such transition happens.

First, we need to discuss how to measure task generalization. To motivate our solution, we take a closer look at the ridge estimator ($\tv_{\text{Bayes}}^{\text{ridge}}$ as in \cref{eq:ridgeestimator}), which is the Bayes-optimal estimator given the ICL task structure. 
We analytically characterize $e_{\text{ICL}}$ and $e_{\text{IDG}}$ errors for this estimator in SI Section \ref{sec:bayesestimatorerror}, giving $$ e^{\text{ridge}}_\text{ICL} = \rho\left(1+\frac{1}{\alpha}\mathcal{M}_{\alpha}\left(\frac{\rho}{\alpha}\right)\right) = e^{\text{ridge}}_\text{IDG}.$$ 
Note that the performance of the ridge estimator is identical on the ICL and IDG test distributions, meaning that it does not rely at all on memorizing task vectors.

Motivated by these arguments, we propose to measure the task generalization capability of a model by studying the difference between its performance on the ICL test distribution and its performance on the IDG test distribution. Specifically, we consider the quantity $g_\text{task}=e_{\text{ICL}}-e_{\text{IDG}}$ for a given model or estimator. This difference being large implies the model, performing better on training tasks, has not learned the true task distribution and is not generalizing in task. Conversely, a small difference between ICL error and IDG error suggests that the model is leveraging the underlying structure of the task distribution rather than overfitting to and interpolating specific task instances seen in training.


Simulations of $g_\text{task}$ as a function of task diversity, $\kappa = k/d$, are shown in \Cref{fig:gtask_linear} for two inference models: dMMSE estimator, and the linear transformer. The dMMSE estimator is included as a benchmark quantifying the performance of a ``pure memorizer'', as it is the Bayes-optimal estimator that assumes task vectors only seen in the training test. We see that, as the task diversity parameter $\kappa$ increases,  $g_\text{task}$ falls for both estimators but at very different rates. 

 

The $g_\text{task}$ for the dMMSE estimator falls very slowly as a function of $\kappa$. This is expected because the explicit form of $\tv_{\text{Bayes}}^{\text{dMMSE}}$, given in \cref{eq:dmmse}, is that of a kernel smoother employing a Gaussian kernel as the weighting function around each task vector. Known results about this class of estimators show the performance of the dMMSE estimator on the ICL task suffers from the ``curse of dimensionality'' \cite{tsybakov2008introduction, belkinrakhlinTsybakov,zavatone2024nadaraya}, i.e. the required number of training samples $k=d\kappa$ would need to be exponential in $d$ to outperform some given $e_\text{ICL}$ error tolerance. Intuitively, the IDG task distribution $\mathcal{P}_{\text{test}}=\mathcal{P}_{\text{train}} = \text{Unif}\{w_1,\cdots,w_k\}, w_i \sim \mathcal{N}(0,I_d)$ and the ICL task distribution $\mathcal{P}_{\text{test}} = \mathcal{N}(0,I_d)$ are indistinguishable (by appropriate measures) only if $k$ is exponentially large in $d$.  
We present more formal arguments in SI Section \ref{sec:bayesestimatorerror}. 

On the other hand, the $g_\text{task}$ for the linear transformer estimator first behaves similarly to the dMMSE estimator, but starts falling sharply after around $\kappa = 1$. How quickly does $g^\text{LT}_\text{task}$ limit to 0 as $\kappa \to \infty$? By expanding $e_\text{ICL}-e_\text{IDG}$ from \eqref{eq:ICLridgeless} and \eqref{eq:IDGridgeless} in $\kappa$ we have the leading asymptotic behavior given by $$g^\text{LT}_\text{task} = 0 + \frac{1}{\kappa}\begin{cases}
    c_1 & \tau < 1 \\
    c_2 & \tau > 1 \\
\end{cases} + \mathcal{O}\left(\frac{1}{\kappa^2}\right)$$ for constants $c_1,c_2$ dependent on $\alpha,\tau,\rho$, provided in SI Section \ref{sec:bayesestimatorerror}.  This sharp fall signals a transition of the model to the task generalization regime. To further understand the nature of this transition and the role of $\kappa$ in the solution learned by the linear attention mechanism, 
consider the regime where $\tau, \alpha \to \infty$ with $\alpha/\tau \equiv c^\ast$ kept fixed. Under this setting, we have
\begin{equation}\label{eq:proportionallimit}
    \lim_{\substack{\tau \to \infty\\\alpha \to \infty}} g^\text{LT}_\text{task} = \begin{cases}  (1-\kappa)\left(1 + \frac{\rho}{1+\rho}c^\ast\right) \qquad &\kappa < 1\\
    0 &\kappa > 1
    \end{cases}\,.
\end{equation} 
This change in analytical behavior indicates a phase transition at $\kappa = 1$. We see this phenomena in both Figures \Cref{fig:gtask_linear}(a.) and \Cref{fig:gtask_linear}(b.) where the $g_\text{task}$ curve plotted against $\kappa$ becomes nondifferentiable at $\kappa = 1$ for $\alpha\to\infty$.

We conclude that the linear transformer model is a much more efficient task generalizer than the dMMSE estimator which is the Bayes-optimal memorizer, as shown in \Cref{fig:gtask_linear}. The $1/\kappa$ decay in $g^\text{LT}_\text{task}$ vs the much slower decay of $g^\text{dMMSE}_\text{task}$ suggests that the linear transformer \textit{quickly} learns an inference algorithm that generalizes in-context rather than interpolates between training tasks.

\section{Experiments with fully-parameterized nonlinear Transformers}\label{sec:experiments}

\begin{figure*}[t!]
    \centering
    \includegraphics{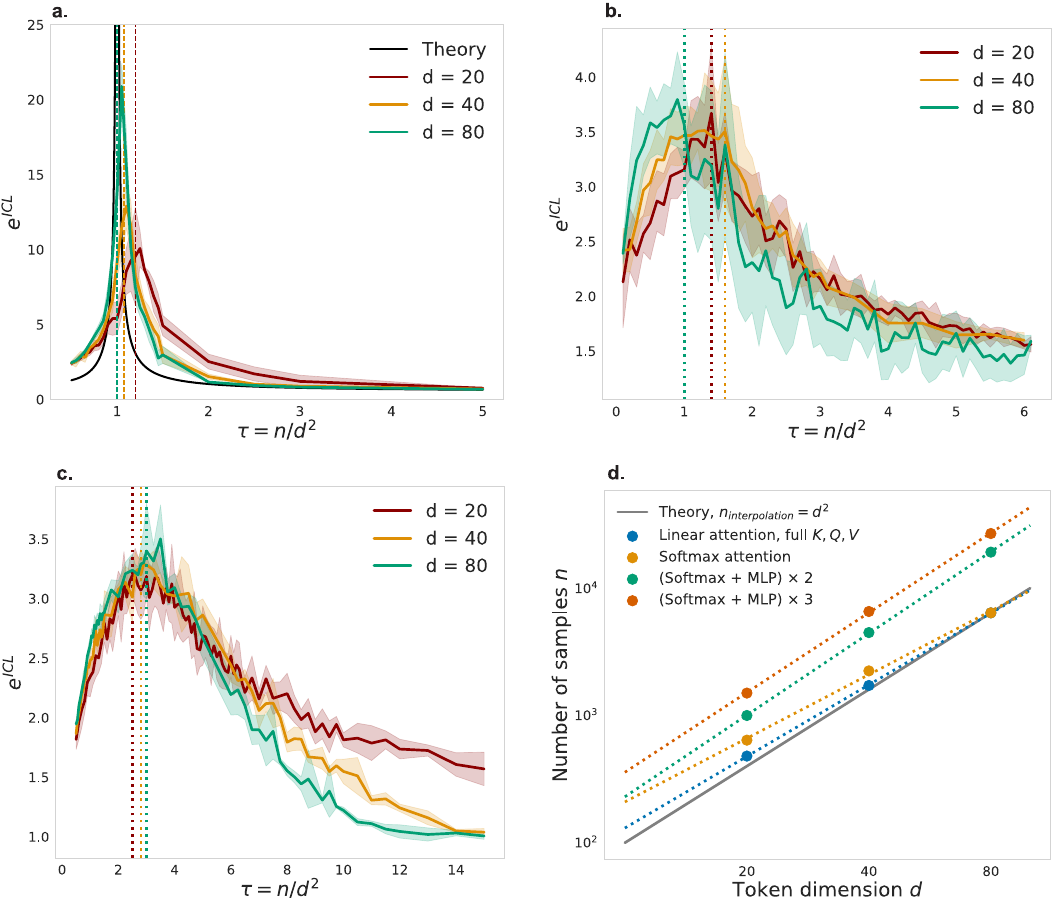}
    \caption{Experimental verification in full linear attention (a.) (full $K,Q,V$ matrices) and nonlinear models (b.) (one block \textit{softmax attention} only), (c.) (two blocks of  \textit{softmax attention and 1-layer MLP}), of both scaling definition for $\tau$ and double descent behavior in $\nsamp$. (a.), (b.), (c.) show error curves against $\tau$ for various architectures, consistent across token dimension $d = 20,40,80$. Some deviations arise near double descent in the linear model (aa), possibly due to parameter inversion or instabilities, and for large $\tau$ in the deep nonlinear model (c.), possibly due to the training procedure. Double-descent phenomena are confirmed: increasing $\nsamp$ will increase error until an interpolation threshold is reached. colored dashed lines indicate experimental interpolation threshold for that architecture and $d$ configuration. (d.) shows that the location of the interpolation threshold occurs for $\nsamp$ proportional to $d^2$, as predicted by the linear theory. Dots are experimental interpolation thresholds for various architectures, and dashed lines are best fit curves correspond to fitting $\log(\nsamp) = a\log(d) +b$, each with $a\approx 2$ (explicitly, $a_{\text{full linear}} = 1.87$, $a_{\text{softmax}} = 1.66$, $a_{\text{2 blocks}} = 2.13$, $a_{\text{3 blocks}} = 2.08$).  Interpolation threshold was computed empirically by searching for location in $\tau$ of sharp increase in value and variance of training error at a fixed number of gradient steps. \textit{Parameters:} $\alpha = 1, \kappa = \infty, \rho = 0.01$. For (a.), (b.), and (c.): variance shown comes from model trained over different samples of pretraining data; lines show averages over 10 runs and shaded region shows standard deviation.}
    \label{fig:nonlineardoubledecent}
\end{figure*}

Given that our theoretical results are derived from a simplified linear attention setting, we aim to determine if these insights can be extended to a nonlinear and deep Transformer models with full $K,Q,V$ parameterization. Even though we would not expect the specific algorithm discussed in \Cref{sec:reduction} to transfer to these nonlinear settings, we will test whether our proposed scalings and main qualitative predictions remain. Specifically, we will investigate the following: (1) whether we have identified the correct scaling regime for non-trivial learning in an in-context learning (ICL) task; (2) whether the full Transformer exhibits a sample-wise double descent, where the location of the peak error scales quadratically with input dimensions as predicted by our theory; and (3) whether the transition from memorization to task generalization occurs. 

\begin{figure}[ht!]
    \centering
        \includegraphics{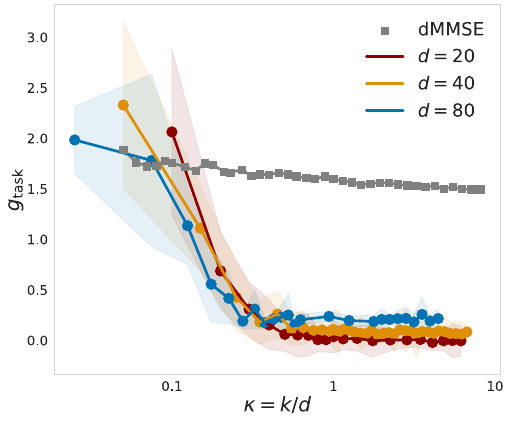}
    \caption{Plot of $g_\text{task}$ against $\kappa$ (loglinear scale) for the nonlinear architecture given in Figure \cref{fig:nonlineardoubledecent}(c.). This demonstrates consistency of $\kappa$ scaling across increasing dimension choices $d = 20,40,80$, as well as a transition in task-generalization measure $g_\text{task}$ near $\kappa = 1$, familiar from the linear theory. \textit{Parameters:} $\tau = 10, \alpha = 1, \rho = 0.01$. Variance shown comes from 10 models trained over different samples of pretraining data.}
    \label{fig:nonlinearbayes}
\end{figure}

Our experiments\footnote{Code to reproduce all experiments is available at \url{https://github.com/Pehlevan-Group/incontext-asymptotics-experiments}.} are done with a standard Transformer architecture, where each sample context initially  takes the form given by \eqref{eq:Zstructure}. The fully-parameterized linear transformer \cref{fig:nonlineardoubledecent}(a.) and the softmax-only transformer \cref{fig:nonlineardoubledecent}(b.) do not use a multilayer perception (MLP) in addition to attention. If MLPs are used (e.g. \cref{fig:nonlineardoubledecent}(c.) and \cref{fig:nonlineardoubledecent}(d.)), the architecture consists of blocks with: (1) a single-head softmax self-attention with $K,Q,V\in\mathbb{R}^{(d+1) \times (d+1)}$ matrices, followed by (2) a two-layer dense MLP with GELU activation and hidden layer of size $d+1$ \cite{vaswani2017attention}. Residual connections are used between the input tokens (padded from dimension $d$ to $d+1$), the pre-MLP output, and the MLP output. We use a variable number of attention+MLP blocks before returning the final logit corresponding to the $(d+1,\cl+1)$th element in the original embedding structure given by \eqref{eq:Zstructure}. The loss function is the mean squared error (MSE) between the predicted label (the output of the model for a given sample $Z$) and the true value $y_{\ell+1}$. We train the model in an offline setting with $n$ total samples $Z_1,\cdots,Z_n$, divided into 10 batches, using the Adam optimizer \cite{kingma2015adam} with a learning rate $10^{-4}$ until the training error converges, typically requiring 10000 epochs\footnote{Note that larger $d$ models are often trained for fewer epochs than smaller $d$ models due to early stopping; that said, whether or not early stopping is used in training does not affect either the alignment of error curves in $d$-scaling nor the qualitative behavior (double descent in $\tau$ and transition in $\kappa$.}. The structure of the pretraining and test distributions exactly follows the setup for the ICL and IDG tasks described in Section \ref{sec:formulation}. 

In \Cref{fig:nonlineardoubledecent}, we plot the average error for a range of fully-parameterized linear (a.) and nonlinear architectures (b. and c.) on the ICL task against the parameter $\load = n/d^2$, which measures the pretraining sample complexity. Notice that the large-$d$ full $K,Q,V$ parameterized linear architecture studied in (a.) follows the general trend of the theory derived on a simplified parameterization. Further, in all figures, plots across different values of $d$ roughly collapse, suggesting again that our scaling relations capture the observed behavior.

Across all architectures, we observe a double descent in $\load$. The peak of this curve occurs at the interpolation threshold, which we identify by tracking the value of $\tau$ at which the training loss becomes non-zero (see figure caption). Our theory predicts that the number of samples $\nsamp$ at the peak, as well as the interpolation threshold, should scale with $d^2$. Indeed, for the fully-parameterized linear transformer in (\cref{fig:nonlineardoubledecent}a.), the double descent occurs very close to $\tau = 1$ for larger $d$, as predicted by the reduced-parameter linear theory. This general quadratic scaling is summarized in (\cref{fig:nonlineardoubledecent}d.) across a range of architectures. These findings suggest that nonlinear Transformers exhibit a similar interpolation transition, with the same quadratic scaling in sample size as seen in the simplified model.

Finally, we can investigate the effects of task diversity on $g_\text{task}$ in a nonlinear model, shown in \Cref{fig:nonlinearbayes}. This provides empirical evidence towards the choice of $\kappa$ scaling, as errors are consistent over a range of token dimension choices. Further, we recover the memorization-to-task-generalization transition in $\kappa$, as observed before in \cite{raventos2023pretraining}. In \Cref{fig:nonlinearbayes}, we see the value of $\kappa$ for which $g_\text{task}$ appears to converge towards 0 is below 1 (as in \eqref{eq:proportionallimit}), suggesting that the more powerful architecture used here may be more ``task-efficient'' than the linear architecture in achieving memorization-to-generalization transition.

\section{Discussion}

In this paper, we computed sharp asymptotics for the in-context learning (ICL) performance in a simplified model of ICL for linear regression using linear attention. This exactly solvable model demonstrates a transition from memorizing pretraining task vectors to generalizing beyond them as the diversity of pretraining tasks increases, echoing empirical findings in full Transformers \cite{raventos2023pretraining}. Additionally, we observe a sample-wise double descent as the amount of pretraining data increases. Our numerical experiments show that full, nonlinear Transformers exhibit qualitatively similar behavior in the scaling regime relevant to our solvable model. Understanding the mechanistic underpinnings of ICL of well-controlled synthetic tasks in solvable models is an important prerequisite to understanding how it emerges from pretraining on natural data \cite{reddy2023mechanistic}.

Our paper falls within a broader program of research that seeks sharp asymptotic characterizations of the performance of machine learning algorithms. This program has a long history in statistical physics \cite{watkin1993rule,engel2001statistical,atanasov2024scaling}, and has in recent years attracted substantial attention in machine learning \cite{hastie2022surprises,gerace2020generalisation,loureiro2021learning,mei2022generalization,hu2022universality,hu2024asymptotics,dhifallah2020precise,atanasov2024scaling, cui2024phase}. For simplicity, we have assumed that the covariates in the in-context regression problem are drawn from an isotropic Gaussian distribution. However, our technical approach could be extended to anisotropic covariates, and, perhaps more interestingly, to featurized linear attention models in which the inputs are passed through some feature map before linear attention is applied \cite{katharopoulos2020transformers,shen2021efficient}. This extension would be possible thanks to an appropriate form of \emph{Gaussian universality}: for certain classes of regression problems, the asymptotic error coincides with that of a model where the true features are replaced with Gaussian features of matched mean and covariance \cite{hastie2022surprises,gerace2020generalisation,loureiro2021learning,mei2022generalization,hu2022universality,hu2024asymptotics,dhifallah2020precise,dubova2023universality,montanari2022universality}. This would allow for a theoretical characterization of ICL for realistic data structure in a closer approximation of full softmax attention, yielding more precise predictions of how performance scales in real Transformers. 

In our analysis, we have assumed that the model is trained to interpolation on a fixed dataset. This allows us to cast our simplified form of linear attention pretraining as a ridge regression problem, which in turn enables our random matrix analysis. In contrast, Transformer-based large language models are usually trained in a nearly-online setting, where each gradient update is estimated using fresh examples with no repeating data \cite{muennighoff2023scaling}. Some of our findings, such as double-descent in the learning curve as a function of the number of pretraining examples, are unlikely to generalize to the fully-online setting. It will be interesting to probe these potential differences in future work. 

Finally, our results have some bearing on the broad question of what architectural features are required for ICL \cite{brown2020language, wei2022emergent, reddy2023mechanistic}. Our work shows that a full Transformer---or indeed even full linear attention---is not required for ICL of linear regression. However, our simplified model retains the structured quadratic pairwise interaction between inputs that is at the heart of the attention mechanism. It is this quadratic interaction that allows the model to solve the ICL regression task, which it does essentially by reversing the data correlation (see \S\ref{sec:reduction}). One would therefore hypothesize that our model is minimal in that further simplifications within this model class would impair its ability to solve this ICL task. In the specific context of regression with isotropic data, a simple point of comparison would be to fix $\Gamma = I_d$, which gives a pretraining-free model that should perform well when the context length is very long. However, this further-reduced model would perform poorly if the covariates of the in-context task are anisotropic. Generally, it would be interesting to investigate when models lacking this precisely-engineered quadratic interaction can learn linear regression in-context \cite{tong2024mlpslearnincontext}, and if they're less sample-efficient than the attention-based models considered here. 

\section*{Acknowledgements}

We thank William L. Tong for helpful discussions regarding numerics. YML was supported by NSF Award CCF-1910410, by the Harvard FAS Dean's Fund for Promising Scholarship, and by a Harvard College Professorship. JAZV and CP were supported by NSF Award DMS-2134157 and NSF CAREER Award IIS-2239780. JAZV is presently supported by a Junior Fellowship from the Harvard Society of Fellows. CP is further supported by a Sloan Research Fellowship. AM acknowledges support from Perimeter Institute, which  is supported in part by the Government of Canada through the Department of Innovation, Science and Economic Development and by the Province of Ontario through the Ministry of Colleges and Universities. This work has been made possible in part by a gift from the Chan Zuckerberg Initiative Foundation to establish the Kempner Institute for the Study of Natural and Artificial Intelligence. This research was supported in part by grants NSF PHY-1748958 and PHY-2309135 to the Kavli Institute for Theoretical Physics (KITP), through the authors' participation in the Fall 2023 program ``Deep Learning from the Perspective of Physics and Neuroscience.''

\textbf{Author Contributions.} Y.M.L. led formal analysis, along with  M.L., J.A.Z-V., A.M., and C.P. M.L. led computational results and verification. C.P. and Y.M.L supervised the project. C.P. conceptualized the project. All authors contributed to key discussions and manuscript drafting.

\bibliography{biblio}

\newpage
\clearpage 
\appendix

\section*{SUPPLEMENTARY INFORMATION}

\noindent Some of the derivations in this document are based on non-rigorous yet technically sound heuristic arguments from random matrix theory. We support these predictions with numerical simulations and discuss the steps required to achieve a fully rigorous proof. All rigorously proven results will be clearly stated as lemmas, propositions, and the like.


\section{Notation}
\label{sec:notation}

\emph{Sets, vectors and matrices:} For each $n \in \N$, $[n] \bydef \set{1, 2, \ldots, n}$. The sphere in $\R^d$ with radius $\sqrt{d}$ is expressed as $\mathcal{S}^{d-1}(\sqrt{d})$. For a vector $v \in \R^{d}$, its $\ell_2$ norm is denoted by $\norm{v}$. For a matrix $A \in \R^{d \times d}$, $\norm{A}_{\mathsf{op}}$ and $\norm{\mA}_\mathsf{F}$ denote the operator (spectral) norm and the Frobenius norm of $\mA$, respectively. Additionally, $\norm{\mA}_\infty \bydef \max_{i,j \in [n]} \abs{A(i,j)}$ denotes the entry-wise $\ell_\infty$ norm. We use $e_1$ to denote the first natural basis vector $(1,0, \ldots, 0)$, and $I$ is an identity matrix. Their dimensions can be inferred from the context. The trace of $A$ is written as $\tr(A)$. 

Our derivations will frequently use the vectorization operation, denoted by $\vecop(\cdot)$. It maps a $d_1 \times d_2$ matrix $A \in \R^{d_1 \times d_2}$ to a vector $v_A = \vecop(A)$ in $\R^{d_1 d_2}$. Note that we shall adopt the \emph{row-major} convention, and thus the rows of $A$ are stacked together to form $v_A$. We also recall the standard identity:
\begin{equation}\label{eq:kronecker_identity}
    \vecop(E_1 E_2 E_3) = (E_1 \otimes E_3^\top) \vecop(E_2),
\end{equation}
where $\otimes$ denotes the matrix Kronecker product, and $E_1, E_2, E_3$ are matrices whose dimensions are compatible for the multiplication operation. For any square matrix $A \in \R^{(L+1)\times (L+1)}$, we introduce the notation
\begin{equation}\label{eq:principal_minor}
    [M]_{\setminus 0} \in \R^{L \times L}
\end{equation}
to denote the principal minor of $M$ after removing its first row and column. 

\emph{Stochastic order notation}: In our analysis, we use a concept of high-probability bounds known as \emph{stochastic domination}. This notion, first introduced in \cite{erdos2013LocalSemicircle,erdos2017Dynamicalapproach}, provides a convenient way to account for low-probability exceptional events where some bounds may not hold.  Consider two families of nonnegative random variables:
\[
X = \big(X^{(d)}(u) : d \in \N, u \in U^{(d)}\big), \quad Y = \big(Y^{(d)}(u) : d \in \N, u \in U^{(d)}\big),
\]
where $U^{(d)}$ is a possibly $d$-dependent parameter set. We say that $X$ is \emph{stochastically dominated} by $Y$, uniformly in $u$, if for every (small) $\varepsilon > 0$ and (large) $D > 0$ we have
\[
\sup_{u \in U^{(d)}} \P[X^{(d)}(u) > d^\varepsilon Y^{(d)}(u)] \le d^{-D}
\]
for sufficiently large $d \ge d_0(\varepsilon, D)$. If $X$ is stochastically dominated by $Y$, uniformly in $u$, we use the notation $X \prec Y$. Moreover, if for some family $X$ we have $\abs{X} \prec Y$, we also write $X = \mathcal{O}_\prec(Y)$. 

We also use the notation $X \simeq Y$ to indicate that two families of random variables $X, Y$ are asymptotically equivalent. Precisely, $X \simeq Y$, if there exists $\varepsilon > 0$ such that for every $D > 0$ we have
\begin{equation}\label{eq:asymp_equiv}
    \P\left[\,\abs{X - Y} > d^{-\varepsilon}\right] \le d^{-D}
\end{equation}
for all sufficiently large $d > d_0(\varepsilon, D)$.

\section{Moment Calculations and Generalization Errors}
\label{sec:moment}
For a given set of parameters $\Gamma$, its generalization error is defined as
\begin{equation}\label{eq:eg_Gamma_def}
    e(\Gamma) = \Eb{\Ptest}{\left(y_{\cl+1} - \inprod{\Gamma, H_Z}\right)^2},
\end{equation}
where $(Z,y_{\cl+1}) \sim \Ptest$ is a new sample drawn from the distribution of the test data set. Recall that $Z$ is the input embedding matrix defined by 
\begin{align}\label{eq:Zstructure_SI}
Z = \left[\begin{array}{ccccc} x_1 & x_2 & \ldots & x_{\cl} & x_{\cl+1} \\ y_1 & y_2 & \ldots & y_{\cl} & 0 \end{array}\right] \in \R^{(d+1)\times(\cl+1)},
\end{align} and $y_{\cl+1}$ denotes the missing value to be predicted. The goal of this section is to  derive an expression for the generalization error $e(\Gamma)$.

Note that the test distribution $\Ptest$ crucially depends on the probability distribution of the task vector $\tv$ used in the linear model $y=w^\top x + \eta$. For the ICL task, we have $\tv \sim \unifsp$, the uniform distribution on the sphere $\sph$. For the IDG task, $\tv$ is sampled uniformly from the set $\set{\tv_1, \tv_2, \ldots, \tv_\ntv}$, where these $\ntv$ vectors are the same as those used in the training data. In what follows, we slightly abuse the notation by writing $\tv \sim \Ptest$ to indicate that $\tv$ is sampled from the task vector distribution associated with $\Ptest$. 

You will note that the distribution $\tv \sim \unifsp$ chosen above does not match the distribution $\tv \sim \mathcal{N}(0,I_d).$ For analytical tractibility and convenience, the remaining calculations will have $\tv \sim \unifsp$ not $\tv \sim \mathcal{N}(0,I_d).$ For large $d$ (in practice, this seems to be any $d>80$), simulations show that these choices are indistinguishable. 

Let $\tv$ be the task vector used in the input matrix $Z$. Throughout the paper, we use $\Eb{\tv}{\cdot}$ to denote the conditional expectation with respect to the randomness in the data vectors $\set{x_i}_{i \in [\cl+1]}$ and the noise $\set{\nrv_i}_{i \in [\cl+1]}$, with the task vector $\tv$ kept fixed. We have the following expressions for the first two \emph{conditional} moments of $(H_Z, y_{\cl+1})$.

\begin{lemma}[Conditional moments]\label{lemma:H_y_moment} Let the task vector $\tv \in \sph$ be fixed. We have
\begin{equation}\label{eq:H_y_first_moment}
    \Eb{\tv}{y_{\cl+1}} = 0, \quad \text{and}\quad \Eb{\tv}{H_Z} = 0.
\end{equation}
Moreover, 
\begin{equation}\label{eq:E_yH}
    \Eb{\tv}{y_{\cl+1} H_Z} = \frac{1}{d} \tv \begin{bmatrix} \tv^\top, & 1 + \rho\end{bmatrix}
\end{equation}
and
\begin{equation}
    \Eb{\tv}{\vecop(H_Z) \vecop(H_Z)^\top} = \frac{(1+\rho)}{d} I_d \otimes \begin{bmatrix}\frac{d}{\cl}I_d + (1+\cl^{-1})(1+\rho)^{-1}\tv \tv^\top & (1+2\cl^{-1}) \tv\\
        (1+2\cl^{-1})\tv^\top & (1+2\cl^{-1})(1+\rho)\end{bmatrix}. \label{eq:E_HH}
\end{equation}
\end{lemma}
\begin{proof}
    Using the equivalent representations in \eqref{eq:H_rep} and \eqref{eq:y_target_rep}, it is straightforward to verify the estimates of the first (conditional) moments in \eqref{eq:H_y_first_moment}. To show \eqref{eq:E_yH}, we note that 
    \begin{equation}
        H_Z = (d/\cl) z_a z_b^\top,
    \end{equation}
    where
    \begin{equation}
        z_a = M_\tv \begin{bmatrix}s \\ u\end{bmatrix} \qquad \text{and}\qquad z_b = \begin{bmatrix}
            M_\tv h \\(\theta_\tv a/\sqrt{d} + \theta_\epsilon)^2 / \sqrt{d} + \theta_q^2 /\sqrt{d}
        \end{bmatrix}.
    \end{equation}
Using the representation in \eqref{eq:y_target_rep}, we have
\begin{equation}
    \Eb{\tv}{y_{\cl+1} H_Z} = (d/\cl) \Eb{\tv}{y_{\cl+1} z_a} \Eb{\tv}{z_b^\top}.
\end{equation}
Computing the expectations $\Eb{\tv}{y_{\cl+1} z_a}$ and $\Eb{\tv}{z_b^\top}$ then gives us \eqref{eq:E_yH}. Next, we show \eqref{eq:E_HH}. Since $z_a$ and $z_b$ are \emph{independent},
    \begin{equation}
        \Ea{\vecop(H_Z) \vecop(H_Z)^\top} = (d/\cl)^2 \, \Ea{z_a z_a^\top} \otimes \Ea{z_b z_b^\top}. \label{eq:E_HH0}
    \end{equation}
    The first expectation on the right-hand side is easy to compute. Since $M_\tv$ is an orthonormal matrix, 
    \begin{equation}
    \Eb{\tv}{z_a z_a^\top} = I_d  \label{eq:Ezaza}   
    \end{equation}
    To obtain the second expectation on the right-hand side of the above expression, we can first verify that
    \begin{equation}
        \Eb{\tv}{M_\tv h h^\top M_\tv }= \frac{\cl}{d^2} \left[(1 + \rho) I_d + \frac{ (\cl+1)}{d} \tv \tv^\top\right]. \label{eq:Ezbzb_1}
    \end{equation}
    Moreover, 
    \begin{equation}
        \Eb{\tv}{M_\tv h \left(( a/\sqrt{d} + \theta_\epsilon)^2 / \sqrt{d} + \theta_q^2 /\sqrt{d}\right)} = \frac{\cl(\cl+2)(1 + \rho)}{d^3}\tv \label{eq:Ezbzb_2}
    \end{equation}
    and
    \begin{equation}
        \Eb{\tv}{\left(( a/\sqrt{d} + \theta_\epsilon)^2 / \sqrt{d} + \theta_q^2 /\sqrt{d}\right)^2} = \frac{\cl(\cl+2)(1 + \rho)^2}{d^3}. \label{eq:Ezbzb_3}
    \end{equation}
Combining \eqref{eq:Ezbzb_1}, \eqref{eq:Ezbzb_2}, and \eqref{eq:Ezbzb_3}, we have
    \begin{equation}
        \Ea{z_b z_b^\top} = \frac{(\cl/d)^2(1+\rho)}{d}\begin{bmatrix}\frac{d}{\cl}I_d + (1+\cl^{-1})(1+\rho)^{-1}\tv\tv^\top & (1+2\cl^{-1}) \tv\\
        (1+2\cl^{-1})\tv^\top & (1+2\cl^{-1})(1+\rho)\end{bmatrix}. \label{eq:Ezbzb}
    \end{equation}
Substituting \eqref{eq:Ezaza} and \eqref{eq:Ezbzb}  into \eqref{eq:E_HH}, we reach the formula in \eqref{eq:E_HH}.
\end{proof}

\begin{proposition}[Generalization error]\label{prop:gen_err} 
For a given weight matrix $\Gamma$, the generalization error of the linear transformer is
\begin{equation}
    \begin{aligned}
    e(\Gamma) &= \frac{1+\rho}{d}\tr\left(\Gamma\begin{bmatrix}\frac{d}{\cl}I_d + (1+\cl^{-1})(1+\rho)^{-1}\Rte & (1+2\cl^{-1}) \bte\\
        (1+2\cl^{-1})\bte^\top & (1+2\cl^{-1})(1+\rho)\end{bmatrix}\Gamma^\top\right)\\
        &\qquad\qquad - \frac{2}{d} \tr\left( \Gamma \begin{bmatrix} \Rte\\(1+\rho)\bte^\top\end{bmatrix}\right) + 1 + \rho,
\end{aligned} \label{eq:population_loss}
\end{equation}
where
\begin{equation}
    \bte \bydef \Eb{\tv \sim \Ptest}{\tv} \quad\text{and}\quad \Rte \bydef \Eb{\tv \sim \Ptest}{\tv \tv^\top}. \label{eq:bte_Rte}
\end{equation}
\end{proposition}
\begin{remark}
        We use $\tv \sim \Ptest$ to indicate that $\tv$ is sampled from the task vector distribution associated with $\Ptest$. For the ICL task, $w \sim \unifsp$. It is straightforward to check that, in this case, 
    \begin{equation}
        \textrm{(ICL)}:\qquad\bte = 0 \quad \text{and} \quad \Rte = I_d. \label{eq:bR_ICL}
    \end{equation}
    For the IDG task, we have 
    \begin{equation}
        \text{(IDG)}:\qquad\bte = \frac{1}{\ntv}\sum_i \tv_i \quad \text{and} \quad \Rte =  \frac{1}{\ntv}\sum_{i \in [\ntv]} \tv_i \tv_i^\top, \label{eq:bR_IDG}
    \end{equation}
    where $\set{\tv_i}_{i \in [\ntv]}$ is the set of fixed task vectors used in the training data.
\end{remark}
\begin{proof}
   Recall the definition of the generalization error in \eqref{eq:eg_Gamma_def}. We start by writing
    \begin{align}
        e( \Gamma) = \vecop(\Gamma)^\top \Ea{\vecop( H_Z) \vecop( H_Z)^\top} \vecop(\Gamma) - 2 \vecop(\Gamma)^\top \vecop(\Ea{y_{N+1} H_Z}) + \Ea{y_{\cl+1}^2},
    \end{align}
    where $H_Z$ is a matrix in the form of \begin{align}
    H_Z \bydef  x_{\cl+1} \begin{bmatrix} \frac{d}{\cl} \sum_{i\leq \cl} y_i x_i^{\top} & \frac{1}{\cl}\sum_{i\leq \cl}y_i^2\end{bmatrix} \in \mathbb{R}^{d\times (d+1)}, \label{eq:H_Z_SI}
\end{align} and $H_Z$ is independent of $\Gamma$. Since $y_{\cl+1} = x_{\cl+1}^\top \tv + \epsilon$, with $\epsilon \sim \mathcal{N}(0, \rho)$ denoting the noise, it is straightforward to check that
\begin{equation}
    \Ea{y_{\cl+1}^2} = 1 + \rho.
\end{equation}
Using the moment estimate \eqref{eq:E_HH} in Lemma~\ref{lemma:H_y_moment} and the identity \eqref{eq:kronecker_identity}, we have
    \begin{equation}
     \begin{aligned}
         &\vecop(\Gamma)^\top \Ea{\vecop( H_Z) \vecop( H_Z)^\top} \vecop(\Gamma)\\
         &\qquad\qquad= \frac{1+\rho}{d}\tr\left(\Gamma\begin{bmatrix}\frac{d}{\cl}I_d + (1+\cl^{-1})(1+\rho)^{-1}\Rte & (1+2\cl^{-1}) \bte\\
        (1+2\cl^{-1})\bte^\top & (1+2\cl^{-1})(1+\rho)\end{bmatrix}\Gamma^\top\right).
     \end{aligned}
     \end{equation}
     Moreover, by \eqref{eq:E_yH},
     \begin{equation}
         \vecop(\Gamma)^\top \vecop\left(\Ea{y_{\cl+1} H_Z}\right) = \frac{1}{d} \tr\left( \Gamma \begin{bmatrix} \Rte\\(1+\rho)\bte^\top\end{bmatrix}\right).
     \end{equation}
\end{proof}

\begin{corollary}\label{cor:eg}
For a given set of parameters $\Gamma$, its generalization error can be written as
\begin{equation}
    e(\Gamma) = \frac 1 d\tr\left(\Gamma \Bte\Gamma^\top\right) - \frac{2}{d} \tr\left( \Gamma \Ate^\top\right) + (1 + \rho) + \mathcal{E}, \label{eq:eg_E_A}
\end{equation}
where
\begin{equation}
        \Ate \bydef \begin{bmatrix}
            \Rte & (1+\rho)\bte
        \end{bmatrix}, \label{eq:Ate}
    \end{equation}
    \begin{equation}
        \Bte \bydef \begin{bmatrix}\frac{1}{\cload}(1+\rho)I_d + \Rte & (1+\rho) \bte\\
        (1+\rho)\bte^\top & (1+\rho)^2\end{bmatrix}, \label{eq:Bte}
    \end{equation}
and $\Rte, \bte$ are as defined in \eqref{eq:bte_Rte}. Moreover, $\mathcal{E}$ denotes an ``error'' term such that
\begin{equation}
    \abs{\mathcal{E}} \le \frac{C_{\cload, \nv}\max\set{\norm{\Rte}_\mathsf{op}, \norm{\bte}, 1}\left(\norm{\Gamma}_\mathsf{F}^2/d\right)}{d}, \label{eq:ge_noise}
\end{equation}
where $C_{\cload, \nv}$ is some constant that only depends on $\cload$ and $\nv$.
\end{corollary}
\begin{proof}
Let
    \begin{equation}
        \Delta = \begin{bmatrix}\frac{d}{\cl}(1+\rho)I_d + (1+\cl^{-1})\Rte & (1+2\cl^{-1})(1+\rho) \bte\\
        (1+2\cl^{-1})(1+\rho)\bte^\top & (1+2\cl^{-1})(1+\rho)^2\end{bmatrix} - \Bte.
    \end{equation}
It is straightforward to check that
\begin{align}
    \mathcal{E} &= \frac 1 d \tr\left(\Gamma \Delta \Gamma^\top \right)\\
    &= \frac 1 d \vecop(\Gamma)^\top (I_d \otimes \Delta) \vecop(\Gamma)\\
    &\le \norm{\Delta}_\mathsf{op} \frac{\norm{\Gamma}_\mathsf{F}^2}{d}.
\end{align}
The bound in \eqref{eq:ge_noise} follows from the estimate that $\norm{\Delta}_\mathsf{op} \le C_{\cload, \nv}\max\set{\norm{\Rte}_\mathsf{op}, \norm{\bte}, 1} /d$.
\end{proof}

\begin{remark}\label{rem:eg}
    Consider the optimal weight matrix $\Gamma^\ast$ obtained by solving the ridge regression problem in
    \begin{align}\label{eq:ridge_LT_SI}
    \Gamma^\ast &= \underset{\Gamma}{\arg\,\min}\, \sum_{\mu =1 }^{\nsamp} \left(y_{\cl+1}^{\mu} - \inprod{\Gamma, H_{Z^{\mu}}} \right)^2 +   \frac{\nsamp}{d}\lambda \norm{\Gamma}_\mathrm{F}^2\,. 
\end{align} Since $\Gamma^\ast$ is the optimal solution of \eqref{eq:ridge_LT_SI}, we must have
    \begin{equation}
        \frac{n}{d}\lambda\norm{\Gamma^\ast}_\mathsf{F}^2 \le \sum_{\mu \in [n]} (y_{\cl+1}^\mu)^2,
    \end{equation}
    where the right-hand side is the value of the objective function of \eqref{eq:ridge_LT_SI} when we choose $\Gamma$ to be the all-zero matrix. It follows that
    \begin{equation}\label{eq:Gamma_norm}
        \frac{\norm{\Gamma^\ast}_\mathsf{F}^2}{d} \le  \frac{\sum_{\mu \in [n]} (y_{\cl+1}^\mu)^2}{\lambda \, n}.
    \end{equation}
    By the law of large numbers, $\frac{\sum_{\mu \in [n]} y_\mu^2}{n} \to  1 + \rho$ as $n \to \infty$. Thus, $\norm{\Gamma^\ast}_\mathsf{F}^2/{d}$ is asymptotically bounded by the constant $(1+\nv)/\lambda$. Furthermore, it is easy to check that $\norm{\Rte}_\mathsf{op} = \mathcal{O}(1)$ and $\norm{\bte} = \mathcal{O}(1)$ for both ICL [see \eqref{eq:bR_ICL}] and IDG [see \eqref{eq:bR_IDG}]. It then follows from Corollary~\ref{cor:eg} that the generalization error associated with the optimal parameers $\Gamma^\ast$ is asymptotically determined by the first three terms on the right-hand side of \eqref{eq:eg_E_A}.
\end{remark}

\section{Analysis of Ridge Regression: Extended Resolvent Matrices}\label{sec:extendedresolvent}

We see from Corollary~\ref{cor:eg} and Remark~\ref{rem:eg} that the two key quantities in determining the generalization error $e(\Gamma^\ast)$ are
\begin{equation}\label{eq:GA_GBG}
        \frac 1 d \tr(\Gamma^\ast \Ate^\top) \qquad \text{and}\qquad \frac 1 d \tr( \Gamma^\ast \Bte (\Gamma^\ast)^\top),
    \end{equation}
where $\Ate$ and $\Bte$ are the matrices defined in \eqref{eq:Ate} and \eqref{eq:Bte}, respectively. In this section, we show that the two quantities in \eqref{eq:GA_GBG} can be obtained by studying a parameterized family of extended resolvent matrices.

To start, we observe that the ridge regression problem in \eqref{eq:ridge_LT_SI} admits the following closed-form solution:
\begin{equation}\label{eq:ridge_sol}
    \vecop(\Gamma^\ast) = G \left(\textstyle\sum_{\mu \in [n]} y_\mu \vecop({H_\mu})\right)/d,
\end{equation}
where $G$ is a resolvent matrix defined as
\begin{equation}\label{eq:resolvent}
    G = \left(\textstyle \sum_{\mu \in [n]} \vecop({H_\mu}) \vecop({H_\mu})^\top/d+ \load\lambda I\right)^{-1}.
\end{equation}
For our later analysis of the generalization error, we need to consider a more general, ``parameterized'' version of $G$, defined as
\begin{equation}\label{eq:resolvent_pM}
    G(\ps) = \left(\textstyle \sum_{\mu \in [n]} \vecop({H_\mu}) \vecop({H_\mu})^\top/d+ \ps\pM + \load\lambda I\right)^{-1},
\end{equation}
where $\pM \in \R^{(d^2+d) \times (d^2+d)}$ is a symmetric positive-semidefinite matrix and $\ps$ is a nonnegative scalar. The original resolvent $G$ in \eqref{eq:resolvent} is a special case, corresponding to $\ps = 0$. 

The objects in \eqref{eq:ridge_sol} and \eqref{eq:resolvent_pM} are the submatrices of an \emph{extended} resolvent matrix, which we construct as follows. For each $\mu \in [n]$, let
\begin{equation}\label{eq:z_mu}
    z_\mu = \begin{bmatrix}y_\mu / d \\ \vecop(H_\mu)/\sqrt d\end{bmatrix}
\end{equation}
be an $(d^2+d+1)$-dimensional vector. Let \begin{equation}\label{eq:Ke}
    \pM_e = \begin{bmatrix} 0 & \\
    & \pM\end{bmatrix},
\end{equation}
where $\pM$ is the $(d^2+d)\times (d^2+d)$ matrix in \eqref{eq:resolvent_pM}.
Define an extended resolvent matrix
\begin{equation}\label{eq:Ge}
    G_e(\ps) = \frac{1}{\sum_{\mu \in [n]} z_\mu z_\mu^\top + \ps \pM_e + \load\lambda I}.
\end{equation}
By block-matrix inversion, it is straightforward to check  that
\begin{equation}\label{eq:Ge_block}
    G_e(\ps) = \begin{bmatrix}
        c(\ps) & -c(\ps)q^\top(\ps) \\
        -c(\ps)q(\ps)  & G(\ps) + c(\ps) q(\ps)q^\top(\ps)
    \end{bmatrix},
\end{equation}
where
\begin{equation}\label{eq:q_Gamma}
    q(\ps) \bydef \frac{1}{d^{3/2}} G(\ps) \left(\textstyle\sum_{\mu \in [n]} y_\mu \vecop({H_\mu})\right)
\end{equation}
is a vector in $\R^{d(d+1)}$, and $c(\ps)$ is a scalar such that
\begin{equation}\label{eq:c_lam_ps}
    \frac{1}{c(\ps)} = \frac{1}{d^2}\sum_{\mu \in [n]} y_\mu^2 + \load\lambda - \frac{1}{d^3} \sum_{\mu, \nu\in [n]} y_\mu y_\nu \vecop(H_\mu)^\top G(\ps) \vecop(H_\nu).
\end{equation}
By comparing \eqref{eq:q_Gamma} with \eqref{eq:ridge_sol}, we see that
\begin{equation}\label{eq:Gamma_q_0}
    \vecop(\Gamma^\ast) = \sqrt{d}\,q(0).
\end{equation}
Moreover, as shown in the following lemma, the two key quantities in \eqref{eq:GA_GBG} can also be obtained from the extended resolvent $G_e(\ps)$.

\begin{lemma}\label{lemma:Gamma_AB}
    For any matrix $A \in \R^{d \times (d+1)}$,
    \begin{equation}\label{eq:GA_Ge}
        \frac 1 d \tr(\Gamma^\ast A^\top) = \frac{-1}{c(0)\sqrt{d}} \begin{bmatrix}
            0 & \vecop(A)^T
        \end{bmatrix} G_e(0) e_1,
    \end{equation}
where $e_1$ denotes the first natural basis vector in $\R^{d^2+d+1}$. Moreover, for any symmetric and positive semidefinite matrix $B \in \R^{(d+1) \times (d+1)}$, if we set 
\begin{equation}\label{eq:pM_B}
    \pM = I_d \otimes B
\end{equation}
in \eqref{eq:Ke}, then
\begin{equation}\label{eq:GBG_Ge}
    \frac 1 d \tr( \Gamma^\ast B (\Gamma^\ast)^\top) = \frac{\diff}{\diff\ps} \left(\frac{1}{c(\ps)}\right) \biggr|_{\ps = 0}.
\end{equation}
\end{lemma}
\begin{proof}
    The identity \eqref{eq:GA_Ge} follows immediately from the block form of $G_e(\ps)$ in \eqref{eq:Ge_block} and the observation in \eqref{eq:Gamma_q_0}. To show \eqref{eq:GBG_Ge}, we take the derivative of $1/c(\ps)$ with respect to $\ps$. From \eqref{eq:c_lam_ps}, and using the identity
    \begin{equation}
        \frac{\diff}{\diff \ps} G(\ps) = -G(\ps) \pM G(\ps),
    \end{equation}
    we have
    \begin{align}
\frac{\diff}{\diff \ps} \left(\frac{1}{c(\ps)}\right) &= \frac 1 {d^3}  \sum_{\mu, \nu\in [n]} y_\mu y_\nu \vecop(H_\mu)^\top G(\ps) \pM G(\ps) \vecop(H_\nu)\\
&= q^\top(\ps) \pM q(\ps).\label{eq:c_partial_derivative}
\end{align}
Thus, by \eqref{eq:Gamma_q_0},
\begin{align}
    \frac{\diff}{\diff\ps} \left(\frac{1}{c(\ps)}\right) \biggr|_{\ps = 0} &= \frac 1 d \left(\vecop(\Gamma^\ast)\right)^\top \pM \vecop(\Gamma^\ast)\\
    &= \frac 1 d \left(\vecop(\Gamma^\ast)\right)^\top \left(I_d \otimes B\right) \vecop(\Gamma^\ast).
\end{align}
Applying the identity in \eqref{eq:kronecker_identity} to the right-hand side of the above equation, we reach \eqref{eq:GBG_Ge}.
\end{proof}

\begin{remark}
    To lighten the notation, we will often write $G_e(\ps)$ [resp. $G(\ps)$] as $G_e$ [resp. $G$], leaving their dependence on the parameter $\ps$ implicit.
\end{remark}

\begin{remark}
    In light of \eqref{eq:pM_B} and \eqref{eq:GBG_Ge}, we will always choose
    \begin{equation}\label{eq:pM_Bte}
        \pM = I_d \otimes \Bte,
    \end{equation}
where $\Bte$ is the matrix defined in \eqref{eq:Bte}.
\end{remark}

\section{An Asymptotic Equivalent of the Extended Resolvent Matrix}
\label{sec:deterministic_equivalent}

In this section, we derive an asymptotic equivalent of the extended resolvent $G_e$ defined in \eqref{eq:Ge}. From this equivalent version, we can then obtain the asymptotic limits of the right-hand sides of \eqref{eq:GA_Ge} and \eqref{eq:GBG_Ge}. Our analysis relies on non-rigorous but technically sound heuristic arguments from random matrix theory. Therefore, we refer to our theoretical predictions as \emph{results} rather than propositions.

Recall that there are $\ntv$ unique task vectors $\set{\tv_i}_{i \in [\ntv]}$ in the training set. Let 
\begin{equation}\label{eq:btr_Rtr}
    \btr \bydef \frac 1 \ntv \sum_{i \in [\ntv]} \tv_i \quad \text{and} \quad \Rtr \bydef \frac 1 \ntv \sum_{i \in [\ntv]} \tv_i \tv_i^\top 
\end{equation}
denote the empirical mean and correlation matrix of these $\ntv$ regression vectors, respectively. Define
\begin{equation}\label{eq:Atr}
        \Atr \bydef \begin{bmatrix}\Rtr & (1+\rho)\btr\end{bmatrix}.
    \end{equation}
and
\begin{equation}\label{eq:Etr}
    \Etr \bydef \begin{bmatrix} \frac{(1+\rho)}{\alpha}I_d +  \Rtr & (1+\rho)\btr\\
    (1+\rho)\btr^\top & (1+\rho)^2\end{bmatrix}.
\end{equation}

\begin{definition}\label{def:equivalent}
    Consider the extended resolvent $G_e(\ps)$ in \eqref{eq:Ge}, with $\pM_e$ chosen in the forms of \eqref{eq:Ke} and \eqref{eq:pM_Bte}. Let $\widetilde G_e$ be another matrix of the same size as $G_e(\ps)$. We say that $\widetilde G_e$ and $G_e(\ps)$ are asymptotically equivalent, if the following conditions hold.
\begin{enumerate}
    \item[(1)] For any two deterministic and unit-norm vectors $u, v \in \R^{d^2+d+1}$, 
    \begin{equation}\label{eq:equiv_uv}
        u^\top G_e(\ps) v \simeq u^\top \widetilde G_e v,
    \end{equation}
    where $\simeq$ is the asymptotic equivalent notation defined in \eqref{eq:asymp_equiv}.
    \item[(2)] Let $\Atr = \begin{bmatrix} \Rtr & (1+\nv) \btr\end{bmatrix}$. For any deterministic, unit-norm vector $v \in \R^{d^2+d+1}$,
    \begin{equation}\label{eq:equiv_Av}
        \frac{1}{\sqrt{d}}\begin{bmatrix} 0 &\vecop(\Atr)^\top\end{bmatrix} G_e(\ps) v \simeq \frac{1}{\sqrt{d}}\begin{bmatrix} 0 &\vecop(\Atr)^\top\end{bmatrix} \widetilde G_e v.
    \end{equation}
    \item[(3)] Recall the notation introduced in \eqref{eq:principal_minor}. We have
    \begin{equation}\label{eq:equiv_trace}
        \frac{1}{d^2} \tr\left(\left[G_e(\ps)\right]_{\setminus 0} \cdot \left[I \otimes \Etr\right]\right) = \frac{1}{d^2} \tr\left(\left[\widetilde G_e\right]_{\setminus 0} \cdot \left[I \otimes \Etr\right]\right) + \mathcal{O}_\prec(d^{-1/2}),
    \end{equation}
    where $\left[G_e(\ps)\right]_{\setminus 0}$ and $\left[\mathcal{G}_e(\ps)\right]_{\setminus 0}$ denote the principal minors of $G_e(\ps)$ and $\mathcal{G}_e(\ps)$, respectively.
\end{enumerate}
\end{definition}

\begin{result}\label{res:equivalent}
Let $\chi_\ps$ denote the unique positive solution to the equation
\begin{equation}\label{eq:chi_Etr}
    \chi_\ps = \frac{1}{d}\tr\left[\Big(\frac{\tau }{1+\chi_\ps} \Etr + \ps \Bte + \lambda \tau  I_d\Big)^{-1} \Etr\right],
\end{equation}
where $\Bte$ is the positive-semidefinite matrix in \eqref{eq:Bte}, with $\bte, \Rte$ chosen accroding to \eqref{eq:bR_ICL} or \eqref{eq:bR_IDG}. The extended resolvent $G_e(\ps)$ in \eqref{eq:Ge} is asymptotically equivalent to
\begin{equation}\label{eq:Ge_equiv_ps}
    \mathcal{G}_e(\ps) \bydef \left(\frac{\load}{1+\chi_\ps}  \begin{bmatrix}
    1+\rho & \frac{1}{\sqrt d} \vecop\left(\begin{bmatrix}\Rtr &  (1+\rho)\btr\end{bmatrix}\right)^\top\\
    \frac{1}{\sqrt d} \vecop\left(\begin{bmatrix}\Rtr &  (1+\rho)\btr\end{bmatrix}\right) & I_d \otimes \Etr
    \end{bmatrix} + \ps \pM_e + \load\lambda I\right)^{-1},
\end{equation}
in the sense of Definition~\ref{def:equivalent}. In the above expression, $\pM_e$ is the matrix in \eqref{eq:Ke} with $\pM = I_d \otimes \Bte$.
\end{result}

In what follows, we present the steps in reaching the asymptotic equivalent $\mathcal{G}_e(\ps)$ given in \eqref{eq:Ge_equiv_ps}. To start, let $G_e^{[\mu]}$ to denote a ``leave-one-out'' version of $G_e$, defined as
\begin{equation}
    G_e^{[\mu]} = \frac{1}{\sum_{\nu \neq \mu} z_\nu z_\nu^\top + \ps \pM_e + \load\lambda I}.
\end{equation}
By \eqref{eq:Ge}, we have
\begin{equation}
    G_e \left(\textstyle\sum_{\mu \in [n]} z_\mu z_\mu^\top + \ps \pM_e + \load\lambda I\right) = I.
\end{equation}
Applying the Woodbury matrix identity then gives us
\begin{equation}\label{eq:Ge_loo_id}
    \sum_{\mu \in [n]} \frac{1}{1 + z_\mu^\top G_e^{[\mu]} z_\mu} G_e^{[\mu]} z_\mu z_\mu^\top + G_e  (\ps \pM_e + \load\lambda I)= I.
\end{equation}
To proceed, we study the quadratic form $z_\mu^\top G_e^{[\mu]} z_\mu$. Let $\tv_\mu$ denotes the task vector associated with $z_\mu$. Conditioned on $\tv_\mu$ and $G_e^\mu$, the quadratic form $z_\mu^\top G_e^{[\mu]} z_\mu$ concentrates around its \emph{conditional expectation} with respect to the remaining randomness in $z_\mu$. Specifically, 
\begin{equation}
    z_\mu^\top G_e^{[\mu]} z_\mu = \chi^\mu(\tv_\mu) + \mathcal{O}_\prec(d^{-1/2}),
\end{equation}
where
\begin{equation}\label{eq:zmu_quadratic}
    \chi^\mu(\tv_\mu) \bydef \frac{1}{d^2} \tr\left(\left[G_e^\mu\right]_{\setminus 0} \cdot \left[I \otimes E(\tv_\mu)\right]\right),
\end{equation}
and 
\begin{equation}\label{eq:Etv}
    E(\tv) \bydef  \begin{bmatrix}\frac{1+\nv}{\cload}I_d + \tv\tv^\top & (1+\nv) \tv\\
        (1+\nv)\tv^\top & (1+\nv)^2\end{bmatrix}.
\end{equation}

Substituting $z_\mu^\top G_e^{[\mu]} z_\mu$ in \eqref{eq:Ge_loo_id} by $\chi^\mu(\tv_\mu)$, we get
\begin{equation}\label{eq:Ge_loo_id_quadratic}
    \sum_{\mu \in [n]} \frac{1}{1 + \chi^\mu(\tv_\mu)} G_e^{[\mu]} z_\mu z_\mu^\top + G_e  (\ps \pM_e + \load\lambda I)= I + \Delta_1,
\end{equation}
where
\begin{equation}\label{eq:Delta1}
    \Delta_1 \bydef \sum_{\mu \in [n]} \frac{z_\mu^\top G_e^{[\mu]} z_\mu-\chi^\mu(\tv_\mu)}{(1 + \chi^\mu(\tv_\mu))(1+z_\mu^\top G_e^{[\mu]} z_\mu)} G_e^{[\mu]} z_\mu z_\mu^\top
\end{equation}
is a matrix that captures the approximation error of the above substitution.

Next, we replace $z_\mu z_\mu^\top$ on the left-hand side of \eqref{eq:Ge_loo_id_quadratic} by its \emph{conditional} expectation $\Eb{\tv_\mu}{z_\mu z_\mu^\top}$, conditioned on the task vector $\tv_\mu$. This allows us to rewrite \eqref{eq:Ge_loo_id_quadratic} as
\begin{equation}\label{eq:Ge_loo_id_expectation0}
    \sum_{\mu \in [n]} \frac{1}{1 + \chi^\mu(\tv_\mu)} G_e^{[\mu]} \Eb{\tv_\mu}{z_\mu z_\mu^\top} + G_e  (\ps \pM_e + \load\lambda I)= I + \Delta_1 + \Delta_2,
\end{equation}
where
\begin{equation}
    \Delta_2 \bydef \sum_{\mu \in [n]} \frac{1}{1 + \chi^\mu(\tv_\mu)} G_e^{[\mu]} \left(\Eb{\tv_\mu}{z_\mu z_\mu^\top} - z_\mu z_\mu^\top \right)
\end{equation}
captures the corresponding approximation error. Recall the definition of $z_\mu$ in \eqref{eq:z_mu}. Using the moment estimates in Lemma~\ref{lemma:H_y_moment}, we have
\begin{equation}\label{eq:Ezz_gen}
    \Eb{\tv_\mu}{z_\mu z_\mu^\top} = \frac{1}{d^2} \begin{bmatrix}
    1+\rho & \frac{1}{\sqrt d} \tv_\mu^\top \otimes \begin{bmatrix} \tv_\mu^\top & 1+\rho\end{bmatrix}\\
    \frac{1}{\sqrt d} \tv_\mu \otimes \begin{bmatrix}\tv_\mu \\ 1+\rho\end{bmatrix} & I_d \otimes E(\tv_\mu)
    \end{bmatrix} + \frac{1}{d^2}\begin{bmatrix}0 & \\
    & I_d \otimes \mathcal{E}_\mu
    \end{bmatrix},
\end{equation}
where $E(\tv_\mu)$ is the matrix defined in \eqref{eq:Etv} and 
\begin{equation}
    \mathcal{E}_\mu = \frac{1}{\cl}\begin{bmatrix}
    \tv_\mu\tv_\mu^\top & 2(1+\nv) \tv_\mu\\
    2(1+\nv) \tv_\mu^\top & 2 (1+\nv)^2
    \end{bmatrix}.
\end{equation}
Replacing the conditional expectation $\Eb{\tv_\mu}{z_\mu z_\mu^\top}$ in \eqref{eq:Ge_loo_id_expectation0} by the main (i.e. the first) term on the right-hand side of \eqref{eq:Ezz_gen}, we can transform \eqref{eq:Ge_loo_id_expectation0} to
\begin{equation}\label{eq:Ge_loo_id_expectation}
    \frac{\load}{n}\sum_{\mu \in [n]} \frac{1}{1 + \chi^\mu(\tv_\mu)} G_e^{[\mu]} \begin{bmatrix}
    1+\rho & \frac{1}{\sqrt d} \tv_\mu^\top \otimes \begin{bmatrix} \tv_\mu^\top & 1+\rho\end{bmatrix}\\
    \frac{1}{\sqrt d} \tv_\mu \otimes \begin{bmatrix}\tv_\mu \\ 1+\rho\end{bmatrix} & I_d \otimes E(\tv_\mu)
    \end{bmatrix} + G_e  (\ps \pM_e + \load\lambda I)= I + \Delta_1 + \Delta_2 + \Delta_3,
\end{equation}
where we recall $\load = n/d^2$, and we use $\Delta_3$ to capture the approximation error associated with $\mathcal{E}_\mu$.

Next, we replace the ``leave-one-out'' terms $G_e^\mu$ and $\chi^\mu(\tv_\mu)$ in \eqref{eq:Ge_loo_id_expectation} by their ``full'' versions. Specifically, we replace $G_e^\mu$ by $G_e$, and $\chi^\mu(\tv_\mu)$ by
\begin{equation}\label{eq:zmu_quadratic_full}
    \chi(\tv_\mu) \bydef \frac{1}{d^2} \tr\left(\left[G_e\right]_{\setminus 0} \cdot \left[I \otimes E(\tv_\mu)\right]\right).
\end{equation}
It is important to note the difference between \eqref{eq:zmu_quadratic} and \eqref{eq:zmu_quadratic_full}: the former uses $G_e^\mu$ and the latter $G_e$. After these replacements and using $\Delta_4$ to capture the approximation errors, we have
\begin{equation}\label{eq:Ge_equiv_0}
    G_e\left(\frac{\load}{n}\sum_{\mu \in [n]} \frac{1}{1 + \chi(\tv_\mu)}  \begin{bmatrix}
    1+\rho & \frac{1}{\sqrt d} \tv_\mu^\top \otimes \begin{bmatrix} \tv_\mu^\top & 1+\rho\end{bmatrix}\\
    \frac{1}{\sqrt d} \tv_\mu \otimes \begin{bmatrix}\tv_\mu \\ 1+\rho\end{bmatrix} & I_d \otimes E(\tv_\mu)
    \end{bmatrix} + \ps \pM_e + \load\lambda I\right)= I + \sum_{j \le 4} \Delta_j.
\end{equation}

Recall that there are $\ntv$ unique task vectors $\set{\tv_i}_{1 \le i \le \ntv}$ in the training set consisting of $n$ input samples. Each sample is associated with one of these task vectors, sampled uniformly from the set $\set{\tv_i}_{1 \le i \le \ntv}$. In our analysis, we shall assume that $\ntv$ divides $n$ and that each unique task vector is associated with exactly $n/\ntv$ input samples. (We note that this assumption merely serves to simplify the notation. The asymptotic characterization of the random matrix $G_e$ remains the same even without this assumption.) Observe that there are only $\ntv$ unique terms in the sum on the left-hand side of \eqref{eq:Ge_equiv_0}. Thus,
\begin{equation}\label{eq:Ge_equiv_1}
    G_e\left(\frac{\load}{\ntv}\sum_{i \in [\ntv]} \frac{1}{1 + \chi(\tv_i)}  \begin{bmatrix}
    1+\rho & \frac{1}{\sqrt d} \tv_i^\top \otimes \begin{bmatrix} \tv_i^\top & 1+\rho\end{bmatrix}\\
    \frac{1}{\sqrt d} \tv_i \otimes \begin{bmatrix}\tv_i \\ 1+\rho\end{bmatrix} & I_d \otimes E(\tv_i)
    \end{bmatrix} + \ps \pM_e + \load\lambda I\right)= I + \sum_{j \le 4} \Delta_j.
\end{equation}

So far, we have been treating the $\ntv$ task vectors $\set{\tv_i}_{i \in [\ntv]}$ as fixed vectors, only using the randomness in the input samples that are associated with the data vectors $\set{x_i^\mu}$. To further simplify our asymptotic characterization, we take advantage of the fact that $\set{\tv_i}_{i \in [\ntv]}$ are independently sampled from $\unifsp$. To that end, we can first show that $\chi(\tv_i)$ in \eqref{eq:zmu_quadratic_full} concentrates around its expectation. Specifically, 
\begin{equation}
    \chi(\tv_i) = \Ea{\frac{1}{d^2} \tr\left(\left[G_e\right]_{\setminus 0} \cdot \left[I \otimes E(\tv_i)\right]\right)} + \mathcal{O}_\prec(d^{-1/2}).
\end{equation}
By symmetry, we must have
\begin{equation}
    \Ea{\frac{1}{d^2} \tr\left(\left[G_e\right]_{\setminus 0} \cdot \left[I \otimes E(\tv_{i})\right]\right)} = \Ea{\frac{1}{d^2} \tr\left(\left[G_e\right]_{\setminus 0} \cdot \left[I \otimes E(\tv_{j})\right]\right)}
\end{equation}
for any $1 \le i < j \le \ntv$. It follows that $
    \abs{\chi(\tv_{i}) - \chi(\tv_{j})} = \mathcal{O}_\prec(d^{-1/2})$, and thus, by a union bound,
\begin{equation}\label{eq:chi_k_chi_n}
    \max_{i \in [\ntv]}\,\abs{\chi(\tv_{k_1}) - \widehat\chi_\mathrm{ave}} = \mathcal{O}_\prec(d^{-1/2}),
\end{equation}
where
\begin{equation}\label{eq:chi_n}
    \widehat\chi_\mathrm{ave} \bydef \frac{1}{\ntv} \sum_{i \in [\ntv]} \chi(\tv_i).
\end{equation}
 Upon substituting \eqref{eq:zmu_quadratic_full} into \eqref{eq:chi_n}, it is straightforward to verify the following characterization of $\widehat\chi_\mathrm{ave}$:
\begin{equation}\label{eq:chi_ave}
    \widehat\chi_\mathrm{ave} = \frac{1}{d^2} \tr\left(\left[G_e\right]_{\setminus 0} \cdot \left[I \otimes \Etr\right]\right).
\end{equation}

The estimate in \eqref{eq:chi_k_chi_n} prompts us to replace the terms $\chi(\tv_i)$ in the right-hand side of \eqref{eq:Ge_equiv_1} by the common value $\widehat\chi_\mathrm{ave}$. As before, we introduce a matrix $\Delta_5$ to capture the approximation error associated with this step. Using the newly introduced notation $\Etr, \btr$ and $\Rtr$ in \eqref{eq:Etr} and \eqref{eq:btr_Rtr}, we can then simplify \eqref{eq:Ge_equiv_1} as
\begin{equation}\label{eq:Ge_equiv_2}
\begin{aligned}
    &G_e \left(\frac{\load}{1+\widehat\chi_\mathrm{ave}}  \begin{bmatrix}
    1+\rho & \frac{1}{\sqrt d} \vecop\left(\begin{bmatrix}\Rtr &  (1+\rho)\btr\end{bmatrix}\right)^\top\\
    \frac{1}{\sqrt d} \vecop\left(\begin{bmatrix}\Rtr &  (1+\rho)\btr\end{bmatrix}\right) & I_d \otimes \Etr
    \end{bmatrix} + \ps \pM_e + \load\lambda I\right)\\
    &= I + \sum_{1 \le j \le 5} \Delta_j.
\end{aligned}
\end{equation}
Define
\begin{equation}\label{eq:G_equiv_hat}
    \widehat{\mathcal{G}}_e(\ps) \bydef \left(\frac{\load}{1+\widehat\chi_\mathrm{ave}}  \begin{bmatrix}
    1+\rho & \frac{1}{\sqrt d} \vecop\left(\begin{bmatrix}\Rtr &  (1+\rho)\btr\end{bmatrix}\right)^\top\\
    \frac{1}{\sqrt d} \vecop\left(\begin{bmatrix}\Rtr &  (1+\rho)\btr\end{bmatrix}\right) & I_d \otimes \Etr
    \end{bmatrix} + \ps \pM_e + \load\lambda I\right)^{-1}.
\end{equation}
Then
\begin{equation}\label{eq:Ge_Gequiv_diff}
    G_e = \widehat{\mathcal{G}}_e(\ps) + \underbrace{ \widehat{\mathcal{G}}_e(\ps) \left(\Delta_1 + \Delta_2 + \Delta_3 + \Delta_4 + \Delta_5\right)}_{\text{approximation errors}}.
\end{equation}

\begin{remark}
We claim that $\widehat{\mathcal{G}}_e$ is asymptotically equivalent to $G_e$, in the sense of Definition~\ref{def:equivalent}. Given \eqref{eq:Ge_Gequiv_diff}, proving this claim requires showing that, for $j = 1, 2, \ldots, 5$,
\begin{subequations}
    \begin{equation}
        u^\top \left(\widehat{\mathcal{G}}_e(\ps) \Delta_j\right) v \simeq 0,
    \end{equation}
    \begin{equation}
        \frac{1}{\sqrt{d}}\begin{bmatrix} 0 &\vecop(\Atr)^\top\end{bmatrix} \left(\widehat{\mathcal{G}}_e(\ps) \Delta_j\right) v  \simeq 0,
    \end{equation}
and
\begin{equation}
    \frac{1}{d^2} \tr\left(\left[\widehat{\mathcal{G}}_e(\ps) \Delta_j\right]_{\setminus 0} \cdot \left[I \otimes \Etr\right]\right) \simeq 0,
\end{equation}
\end{subequations}
for any deterministic and unit-norm vectors $u, v$ and for $\Atr = \begin{bmatrix} \Rtr & (1+\nv) \btr\end{bmatrix}$.
\end{remark}

We note the equivalent matrix $\widehat{\mathcal{G}}_e(\ps)$ still involves one scalar $\widehat{\chi}_\mathrm{ave}$ that depends on the original resolvent $G_e(\ps)$. Next, we show that $\widehat{\chi}_\mathrm{ave}$ can be replaced by $\chi_\ps$, the unique positive solution to \eqref{eq:chi_Etr}. To that end, we recall the characterization in \eqref{eq:chi_ave}. Using the claim that $G_e(\ps)$ and $\widehat{\mathcal{G}}_e(\ps)$ are asymptotically equivalent (in particular, in the sense of \eqref{eq:equiv_trace}), we have 
\begin{equation}\label{eq:chi_ave_Ge_hat}
  \widehat{\chi}_\mathrm{ave} \simeq   \frac{1}{d^2} \tr\left(\left[\widehat{\mathcal{G}}_e(\ps) \right]_{\setminus 0} \cdot \left[I \otimes \Etr\right]\right).
\end{equation}
To compute the first term on the right-hand side of the above estimate, we directly invert the block matrix $\widehat{\mathcal{G}}_e(\ps)$ in \eqref{eq:G_equiv_hat}. Recall that $\pM_e$ is chosen in the forms of \eqref{eq:Ke} and \eqref{eq:pM_B}. It is then straightforward to verify that
\begin{equation}\label{eq:Ge_hat_block}
    \widehat{\mathcal{G}}_e = \begin{bmatrix}
        \bar{c} & -\bar{c}\, {\bar q}^\top\\
        -\bar{c} \, \bar{q} & I \otimes F_E(\widehat{\chi}_\mathrm{ave})+ \bar c\, \bar{q} {\bar q}^\top
    \end{bmatrix},
\end{equation}
where $F_E(\chi)$ is a matrix valued function such that
\begin{equation}\label{eq:F_E}
    F_E(\chi) = \Big(\frac{\tau  }{1+\chi} \Etr + \ps B + \lambda \tau I_{d+1}\Big)^{-1},
\end{equation}
\begin{equation}
    \bar q = \frac{\tau}{(1+\widehat{\chi}_\mathrm{ave})\sqrt{d}}\vecop\left(\begin{bmatrix}\Rtr &  (1+\rho)\btr\end{bmatrix}F_E(\widehat{\chi}_\mathrm{ave})\right),
\end{equation}
and
\begin{equation}
    {1/}\bar c = \frac{\tau(1+\rho)}{1+\widehat{\chi}_\mathrm{ave}} + \lambda \tau- \frac{\tau^2}{(1+\widehat{\chi}_\mathrm{ave})^2 d}\tr\left(\begin{bmatrix}\Rtr &  (1+\rho)\btr\end{bmatrix}F_E(\widehat{\chi}_\mathrm{ave})\begin{bmatrix}\Rtr &  (1+\rho)\btr\end{bmatrix}^\top\right).
\end{equation}

Using \eqref{eq:Ge_hat_block}, we can now write the equation \eqref{eq:chi_ave_Ge_hat} as
\begin{equation}\label{eq:chi_LT_1}
\begin{aligned}
    \widehat{\chi}_\mathrm{ave} &\simeq \frac{1}{d}\tr\left(F_E(\widehat{\chi}_\mathrm{ave}) \Etr\right)\\
    &\quad+ \frac{\bar c\, \tau^2}{(1+\widehat{\chi}_\mathrm{ave})^2 d^3} \tr\left(\begin{bmatrix}\Rtr &  (1+\rho)\btr\end{bmatrix}F_E(\widehat{\chi}_\mathrm{ave}) \Etr F_E(\widehat{\chi}_\mathrm{ave}) \begin{bmatrix}\Rtr &  (1+\rho)\btr\end{bmatrix}^\top\right).
\end{aligned}
\end{equation}
The second term on the right-hand side of \eqref{eq:chi_LT_1} is negligible. Indeed, 
\begin{equation}
\begin{aligned}
    &\tr\left(\begin{bmatrix}\Rtr &  (1+\rho)\btr\end{bmatrix}F_E(\widehat{\chi}_\mathrm{ave}) \Etr F_E(\widehat{\chi}_\mathrm{ave}) \begin{bmatrix}\Rtr &  (1+\rho)\btr\end{bmatrix}^\top\right)\\ &\qquad\qquad\qquad\le \norm{F_E(\widehat{\chi}_\mathrm{ave}) \Etr F_E(\widehat{\chi}_\mathrm{ave})}_\mathsf{op} (\, \norm{\Rtr}_\mathsf{F}^2 + (1+\rho)^2 \norm{\btr}^2).
\end{aligned}
\end{equation}
By construction, $\norm{F_E(\widehat{\chi}_\mathrm{ave})}_\mathsf{op} \le (\lambda \tau)^{-1}$. Moreover, since the task vectors $\set{\tv_i}_{i \in [\ntv]}$ are independent vectors sampled from $\unifsp$, it is easy to verify that
\begin{equation}
    \norm{\Etr}_\mathsf{op} = \mathcal{O}_\prec(1), \qquad \norm{\Rtr}_\mathsf{F} = \mathcal{O}_\prec(\sqrt{d}) \qquad \text{and}\qquad \norm{\btr}_2 = \mathcal{O}_\prec(1).
\end{equation}
Finally, since $\bar c$ is an element of $\widehat{\mathcal{G}}_e$, we must have $\abs{\bar c} \le \norm{\widehat{\mathcal{G}}_e}_\mathsf{op} \le (\tau \lambda)^{-1}$. Combining these estimates gives us
\begin{equation}
    \frac{\bar c\, \tau^2}{(1+\widehat{\chi}_\mathrm{ave})^2 d^3} \tr\left(\begin{bmatrix}\Rtr &  (1+\rho)\btr\end{bmatrix}F_E(\widehat{\chi}_\mathrm{ave}) \Etr F_E(\widehat{\chi}_\mathrm{ave}) \begin{bmatrix}\Rtr &  (1+\rho)\btr\end{bmatrix}^\top\right) = \mathcal{O}_\prec(d^{-2}),
\end{equation}
and thus we can simplify \eqref{eq:chi_LT_1} as
\begin{equation}\label{eq:chi_LT_2}
    \widehat{\chi}_\mathrm{ave} \simeq \frac{1}{d}\tr\left[\Big(\frac{\tau }{1+\widehat{\chi}_\mathrm{ave}} \Etr + \ps B + \lambda \tau  I_d\Big)^{-1} \Etr\right].
\end{equation}
Observe that \eqref{eq:chi_LT_2} is a small perturbation of the self-consistent equation in \eqref{eq:chi_Etr}. By the stability of the equation \eqref{eq:chi_Etr}, we then have
\begin{equation}\label{eq:chi_diff}
    \widehat{\chi}_\mathrm{ave} \simeq \chi_\ps,
\end{equation}
where $\chi_\ps$ is the unique positive solution to \eqref{eq:chi_Etr}.

Recall the definitions of $\mathcal{G}_e(\ps)$ and $\widehat{\mathcal{G}}_e(\ps)$ in \eqref{eq:G_equiv_hat} and \eqref{eq:Ge_equiv_ps}, respectively. By the standard resolvent identity,
\begin{equation}\label{eq:G_Ghat_diff}
\begin{aligned}
    &\widehat{\mathcal{G}}_e(\ps) - \mathcal{G}_e(\ps)\\ &= \frac{\tau[\widehat{\chi}_\mathrm{ave} - \chi_\ps]}{[1+\chi_\ps][1+\widehat{\chi}_\mathrm{ave}]}\widehat{\mathcal{G}}_e(\ps) \begin{bmatrix}
    1+\rho & \frac{1}{\sqrt d} \vecop\left(\begin{bmatrix}\Rtr &  (1+\rho)\btr\end{bmatrix}\right)^\top\\
    \frac{1}{\sqrt d} \vecop\left(\begin{bmatrix}\Rtr &  (1+\rho)\btr\end{bmatrix}\right) & I_d \otimes \Etr
    \end{bmatrix}\mathcal{G}_e(\ps).
\end{aligned}
\end{equation}
By construction, $\norm{\widehat{\mathcal{G}}_e(\ps)}_\mathsf{op} \le 1/(\tau \lambda)$ and $\norm{\mathcal{G}_e(\ps)}_\mathsf{op} \le 1/(\tau \lambda)$. Moreover, $\norm{\Etr}_\mathsf{op} \prec 1$ and 
\begin{equation}
\norm{\frac{1}{\sqrt d} \vecop\left(\begin{bmatrix}\Rtr &  (1+\rho)\btr\end{bmatrix}\right)} \prec 1.
\end{equation}
It then follows from \eqref{eq:chi_diff} and \eqref{eq:G_Ghat_diff} that
\begin{equation}\label{eq:G_Ghat_diff_bnd}
    \norm{\widehat{\mathcal{G}}_e(\ps) - \mathcal{G}_e(\ps)}_\mathsf{op} \simeq 0.
\end{equation}
If $\widehat{\mathcal{G}}_e(\ps)$ satisfies the equivalent conditions \eqref{eq:equiv_uv}, \eqref{eq:equiv_Av} and \eqref{eq:equiv_trace} (as claimed in our analysis above), then the estimate in \eqref{eq:G_Ghat_diff_bnd} allows us to easily check that $\mathcal{G}_e(\ps)$ also satisfies \eqref{eq:equiv_uv}, \eqref{eq:equiv_Av} and \eqref{eq:equiv_trace}. Thus, we claim that $\mathcal{G}_e(\ps)$ is asymptotically equivalent to the extended resolvent matrix $G_e(\ps)$ in the sense of Definition~\ref{def:equivalent}.

\section{Asymptotic Limits of the Generalization Errors}\label{sec:AsymptoticLimits}

In this section, we use the characterization in Result~\ref{res:equivalent} to derive the asymptotic limits of the generalization errors of associated with the set of parameters $\Gamma^\ast$ learned from ridge regression. 

\subsection{Asymptotic Limits of the Linear and Quadratic Terms}

From Corollary~\ref{cor:eg} and the discussions in Remark~\ref{rem:eg}, characterizing the test error $e(\Gamma^\ast)$ boils down to computing the linear term $\frac{1}{d} \tr\left( \Gamma^\ast \Ate^\top\right)$ and the quadratic term $\frac 1 d\tr\left(\Gamma^\ast \Bte(\Gamma^\ast)^\top\right)$, where $\Ate$ and $\Bte$ are the matrices defined in \eqref{eq:Ate} and \eqref{eq:Bte}, respectively. 

We consider two different types of test data distributions $\Ptest$: ICL and IDG. [See the main text for details.] From \eqref{eq:bR_ICL} and \eqref{eq:bR_IDG}, these two settings correspond to choosing
\begin{equation}\label{eq:AteBte_ICL}
    \text{(ICL)}:\qquad\Ate = \begin{bmatrix}
        I_d & 0 
    \end{bmatrix} \quad \text{and}\quad \Bte = \begin{bmatrix}(\frac{1+\rho}{\alpha} + 1) I_d  & \\
        & (1+\rho)^2\end{bmatrix}.
\end{equation}
and
\begin{equation}\label{eq:AteBte_IDG}
  \text{(IDG)}:\qquad  \Ate = \Atr \quad \text{and} \quad \Bte = \Etr,
\end{equation}
respectively. In \eqref{eq:AteBte_IDG}, $\Atr$ and $\Etr$ are the matrices defined in \eqref{eq:Atr} and \eqref{eq:Etr}. 

\begin{result}\label{res:Ge_AB}
    Let $\Gamma^\ast$ be the set of parameters learned from the ridge regression problem in \eqref{eq:ridge_LT_SI}. Let $\Ate \in \R^{d \times (d+1)}$ and $\Bte \in \R^{(d+1) \times (d+1)}$ be two matrices constructed as in \eqref{eq:AteBte_ICL} or \eqref{eq:AteBte_IDG}. We have 
    \begin{equation}\label{eq:Gamma_A_trace}
        \frac 1 d \tr(\Gamma^\ast \Ate^\top) \simeq \frac{1}{d}\tr\left(\GammaEq \Ate^\top\right),
    \end{equation}
    and
    \begin{equation}\label{eq:GammaBGamma}
    \begin{aligned}
        &\frac 1 d \tr(\Gamma^\ast \Bte ( \Gamma^\ast)^\top) \simeq \frac 1 d \tr(\GammaEq \Bte (\GammaEq)^T)- \frac{c_e}{d} \tr\left(\Bte\left[(\Etr + \xi I)^{-1}- \xi(\Etr + \xi I)^{-2}\right]\right).
    \end{aligned}
    \end{equation}
    In the above displays, $\GammaEq$ is an asymptotic equivalent of $\Gamma^\ast$, defined as
    \begin{equation}\label{eq:Gamma_e}
    \begin{aligned}
        \GammaEq &\bydef  \begin{bmatrix} \Rtr & (1+\rho)\btr\end{bmatrix} (\Etr + \xi I)^{-1},
    \end{aligned}
    \end{equation}
     where $\xi$ is the unique positive solution to the self-consistent equation
\begin{equation}\label{eq:xi_equation}
    \xi\mathcal{M}_\tload\left(\frac{1+\nv}{\alpha}+\xi\right) - \frac{\tau \lambda}{\xi} = 1 - \tau,
\end{equation}
and $\mathcal{M}_\kappa(\cdot)$ is the function defined in \eqref{eq:m_Wishart}. For convenience, this can be recast so that $\xi$ is the positive \textit{double} root of $$
\kappa \tau \xi^4 + 
(-\kappa \lambda \tau + \kappa \tau x + \kappa \tau - \kappa x - \tau^2 - \kappa + \tau) \xi^3 + 
(-\kappa \lambda \tau x - \kappa \lambda \tau + 2 \lambda \tau^2 - \tau^2 x - \lambda \tau + 2 \tau x - x) \xi^2 + 
(-\lambda^2 \tau^2 + 2 \lambda \tau^2 x - 2 \lambda \tau x) \xi -\lambda^2 \tau^2 x.$$
Moreover, the scalar $c_e$ in \eqref{eq:GammaBGamma} is defined as
\begin{equation}\label{eq:GammaBGamma_c}
    c_e = \frac{\rho +\nu - \nu^2 \mathcal{M}_\tload(\nu)- \xi\left[1-2\nu\mathcal{M}_\tload(\nu) - \nu^2 \mathcal{M}'_\tload(\nu)\right]}{1 - {2\xi}\mathcal{M}_\tload(\nu) - 
    \xi^2 \mathcal{M}'_\tload(\nu) - \tau},
\end{equation}
where
\begin{equation}\label{eq:nu_xi}
    \nu \bydef \frac{1+\nv}{\cload} + \xi.
\end{equation}
\end{result}

To derive the asymptotic characterizations \eqref{eq:Gamma_A_trace} and \eqref{eq:GammaBGamma} in Result~\ref{res:Ge_AB}, we first use block-matrix inversion to rewrite $\mathcal{G}_e(\ps)$ in \eqref{eq:Ge_equiv_ps} as
\begin{equation}\label{eq:Ge_equiv_block}
    {\mathcal{G}}_e(\ps) = \begin{bmatrix}
        c^\ast(\ps) & -c^\ast(\ps)\, (q^\ast(\ps))^\top\\
        -c^\ast(\ps) \, q^\ast(\ps) & I \otimes F_E(\chi_\ps)+ c^\ast(\ps) q^\ast(\ps) (q^\ast(\ps))^\top
    \end{bmatrix},
\end{equation}
where $F_E(\cdot)$ is the matrix-valued function defined in \eqref{eq:F_E}, \emph{i.e.},
\begin{equation}
    F_E(\chi_\ps) = \Big(\frac{\tau  }{1+\chi_\ps} \Etr + \ps \Bte + \lambda \tau I_{d+1}\Big)^{-1}.
\end{equation}
Moreover,
\begin{equation}\label{eq:qast_ps}
    q^\ast(\ps) = \frac{\tau}{(1+\chi_\ps)\sqrt{d}}\vecop\left(\begin{bmatrix}\Rtr &  (1+\rho)\btr\end{bmatrix}F_E(\chi_\ps)\right),
\end{equation}
and
\begin{equation}\label{eq:cinv_ast_ps}
    \frac{1}{c^\ast(\ps)} = \frac{\tau(1+\rho)}{1+\chi_\ps} + \lambda \tau- \frac{\tau^2}{(1+\chi_\ps)^2 d}\tr\left(\begin{bmatrix}\Rtr &  (1+\rho)\btr\end{bmatrix}F_E(\chi_\ps)\begin{bmatrix}\Rtr &  (1+\rho)\btr\end{bmatrix}^\top\right).
\end{equation}
Observe that there is a one-to-one correspondence between the terms in \eqref{eq:Ge_equiv_block} and those in \eqref{eq:Ge_block}. 

To derive the asymptotic characterization given in \eqref{eq:Gamma_A_trace}, we note that
\begin{align}
    \frac 1 d \tr(\Gamma^\ast \Ate^\top) &\simeq \frac{-1}{c(0)\sqrt{d}} \begin{bmatrix}
            0 & \vecop(\Ate)^T
        \end{bmatrix} \mathcal{G}_e(0) e_1 \label{eq:Gamma_A_trace_0}\\
        &= \frac{c^\ast(0)}{c(0)}\cdot \frac 1 d \tr\left(\begin{bmatrix} \Rtr & (1+\rho)\btr\end{bmatrix} (\Etr + \lambda(1+\chi_0) I)^{-1} \Ate^\top\right)\label{eq:Gamma_A_trace_1}\\
        &\simeq \frac 1 d \tr\left(\begin{bmatrix} \Rtr & (1+\rho)\btr\end{bmatrix} (\Etr + \lambda(1+\chi_0) I)^{-1} \Ate^\top\right).\label{eq:Gamma_A_trace_2}
\end{align}
In the above display, \eqref{eq:Gamma_A_trace_0} follows from \eqref{eq:GA_Ge} and the asymptotic equivalence between $G_e(0)$ and $\mathcal{G}_e(0)$. The equality in \eqref{eq:Gamma_A_trace_1} is due to \eqref{eq:Ge_equiv_block} and \eqref{eq:qast_ps}. To reach \eqref{eq:Gamma_A_trace_2}, we note that $c(0) = e_1^\top G_e(0) e_1$ and $c^\ast(0) = e_1^\top \mathcal{G}_e(0) e_1$. Thus, $c(0) \simeq c^\ast(0)$ due to the asymptotic equivalence between $G_e(0)$ and $\mathcal{G}_e(0)$. In Appendix~\ref{appendix:Wishart}, we show that
\begin{equation}\label{eq:xi0_xi}
    \lambda(1+\chi_0) \simeq \xi,
\end{equation}
where $\xi$ is the scalar defined in \eqref{eq:xi_equation}. The asymptotic characterization given in \eqref{eq:Gamma_A_trace} then follows from \eqref{eq:Gamma_A_trace_2} and from the definition of $\GammaEq$ given in \eqref{eq:Gamma_e}.

Next, we use \eqref{eq:GBG_Ge} to derive the asymptotic characterization of the quadratic term in \eqref{eq:GammaBGamma}. Taking the derivative of \eqref{eq:cinv_ast_ps} gives us
\begin{align}
       &\frac{\diff}{\diff \ps}\left(\frac{1}{c^\ast(\ps)}\right) \biggr|_{\ps = 0} = \frac 1 d \tr(\GammaEq \Bte (\GammaEq)^\top) \nonumber\\
       &\qquad- \frac{\tau \chi'_0}{(1+\chi_0)^2}\left(1+\rho - \frac{2}{d} \tr(A_\mathrm{tr} (\Etr + \xi I)^{-1}A_\mathrm{tr}^T) + \frac{1}{d} \tr(A_\mathrm{tr} (\Etr +\xi I)^{-1} \Etr (\Etr +\xi I)^{-1} A_\mathrm{tr}^T)\right)\\
       &= \frac 1 d \tr(\GammaEq \Bte (\GammaEq)^\top) - \frac{\tau \chi'_0}{(1+\chi_0)^2}\left(1+\rho - \frac{1}{d} \tr(\GammaEq A_\mathrm{tr}^T) - \frac{\xi}{d} \tr(\GammaEq (\GammaEq)^\top)\right),\label{eq:diff_cinv_ast}
\end{align}
where $\Atr$ is the matrix defined in \eqref{eq:Atr}. In reaching the above expression, we have also used the estimate in \eqref{eq:xi0_xi}. 

To further simplify our formula, we note that
\begin{equation}\label{eq:Atr_S}
\Atr = S \left(\Etr + \xi I_{d+1} - \Big(\frac{1+\rho}{\alpha} + \xi\Big) I_{d+1}\right),  
\end{equation}
where $S$ is a $d \times (d+1)$ matrix obtained by removing the last row of $I_{d+1}$. Using this identity, we can rewrite the matrix $\GammaEq$ in \eqref{eq:Gamma_e} as
\begin{align}\label{eq:Gamma_e_S}
    \GammaEq &=  S \left(I - \Big(\frac{1+\rho}{\alpha}+\xi\Big) (\Etr + \xi I)^{-1}\right)\\
    &= \begin{bmatrix} I - \nu F_R(\nu) - a^\ast(1+\rho)^2\nu F_R(\nu) \btr \btr^\top F_R(\nu) & a^\ast(1+\rho)\nu F_R(\nu) \btr\end{bmatrix},\label{eq:Gamma_e_precise}
\end{align}
where $F_R(\cdot)$ is the function defined in \eqref{eq:F_R}, and $\nu$ is the parameter given in \eqref{eq:nu_xi}. The second equality \eqref{eq:Gamma_e_precise} is obtained from the explicit formula for $(\Etr + \xi I)^{-1}$ in \eqref{eq:Etr_inv}.

From \eqref{eq:Atr_S} and \eqref{eq:Gamma_e_S}, it is straightforward to check that
\begin{equation}\label{eq:Gamma_e_Atr}
    \frac{1}{d} \tr(\GammaEq A_\mathrm{tr}^T) = 1 - \nu + \nu^2 \frac 1 d \tr(S (\Etr + \xi I)^{-1}S^\top),
\end{equation}
and
\begin{equation}
    \frac{\xi}{d} \tr(\GammaEq (\GammaEq)^\top) = \xi\left[1 - 2\nu \frac 1 d \tr(S (\Etr + \xi I)^{-1}S^\top) + \nu^2 \frac 1 d \tr(S (\Etr + \xi I)^{-2}S^\top\right].
\end{equation}
By using the asymptotic characterizations given in \eqref{eq:GEtr_asymp} and \eqref{eq:GEtr2_asymp}, we then have
\begin{equation}\label{eq:Gamma_e_Atr_asymp}
    \frac{1}{d} \tr(\GammaEq A_\mathrm{tr}^T) \simeq 1 - \nu + 
    \nu^2 \mathcal{M}_\tload(\nu),
\end{equation}
and
\begin{equation}\label{eq:Gamma_e_Gamma_e_asymp}
    \frac{\xi}{d} \tr(\GammaEq (\GammaEq)^\top) \simeq \xi\left[1 - 2\nu \mathcal{M}_\tload(\nu) - \nu^2 \mathcal{M}'_\tload(\nu)\right].
\end{equation}
Substituting \eqref{eq:Gamma_e_Atr_asymp}, \eqref{eq:Gamma_e_Gamma_e_asymp}, and \eqref{eq:dchi} into \eqref{eq:diff_cinv_ast} yields
\begin{equation}
       \frac{\diff}{\diff \ps}\left(\frac{1}{c^\ast(\ps)}\right) \biggr|_{\ps = 0}  \simeq \frac 1 d \tr(\GammaEq \Bte (\GammaEq)^T)- \frac{c_e}{d} \tr\left(\Bte\left[(\Etr + \xi I)^{-1}- \xi(\Etr + \xi I)^{-2}\right]\right),
\end{equation}
where $c_e$ is the scalar defined in \eqref{eq:GammaBGamma_c}. The asymptotic characterization of the quadratic term in \eqref{eq:GammaBGamma} then follows from \eqref{eq:GBG_Ge} and the claim that
\begin{equation}
    \frac{\diff}{\diff \ps}\left(\frac{1}{c(\ps)}\right) \biggr|_{\ps = 0} \simeq \frac{\diff}{\diff \ps}\left(\frac{1}{c^\ast(\ps)}\right) \biggr|_{\ps = 0}.
\end{equation}

\subsection{The Generalization Error of In-Context Learning}\label{sec:icl_generalisation_theory}
\begin{result}\label{res:eg_ICL}
Consider the test distribution $\Ptest$ associated with the ICL task. We have
\begin{equation}\label{eq:eicl_limit}
    e(\Gamma^\ast) \simeq \eicl(\load, \cload,\tload, \nv, \lambda),
\end{equation}
where
\begin{equation}\label{eq:eg_ICL}
\begin{aligned}
    \eicl(\load, \cload,\tload, \nv, \lambda)&\bydef \left(\frac{1+\rho}{\alpha} + 1\right)\left(1 -2\nu\mathcal{M}_\tload(\nu) -\nu^2\mathcal{M}'_\tload(\nu) - c_e\left[\mathcal{M}_\tload(\nu) + \xi \mathcal{M}'_\tload(\nu)\right]\right)\\
    &\qquad\qquad\qquad- 2\left[1 - \nu \mathcal{M}_\tload(\nu)\right] + 1 + \rho,
\end{aligned}
\end{equation}
and $c_e$ is the constant given in \eqref{eq:GammaBGamma_c}.
\end{result}
\begin{remark}
    Recall the definition of the asymptotic equivalence notation ``$\simeq$'' introduced in \sref{notation}. The characterization given in \eqref{eq:eicl_limit} implies that, as $d \to \infty$, the generalization error $e(\Gamma^\ast)$ converges almost surely to the deterministic quantity $\eicl(\load, \cload,\tload, \nv, \lambda)$.
\end{remark}

To derive \eqref{eq:eicl_limit}, our starting point is the estimate
\begin{equation}\label{eq:eg_AB}
    e(\Gamma^\ast) \simeq \frac{1}{d}\tr\left(\Gamma^\ast B_\mathrm{test} (\Gamma^\ast)^\top\right)- \frac{2}{d} \tr\left(\Gamma^\ast A_\mathrm{test}^\top\right) + 1 + \rho,
\end{equation}
which follows from Corollary~\ref{cor:eg} and the discussions in Remark~\ref{rem:eg}. We consider the ICL task here, and thus $\Ate$ and $\Bte$ are given in \eqref{eq:AteBte_ICL}. The asymptotic limits of the first two terms on the right-hand side of the above equation can be obtained by the characterizations given in Result~\ref{res:Ge_AB}.  

Using \eqref{eq:Gamma_A_trace} and the expressions in \eqref{eq:Gamma_e_precise} and \eqref{eq:AteBte_ICL}, we have
\begin{align}
    \frac 1 d \tr(\Gamma^\ast \Ate^\top) &\simeq \frac{1}{d}\tr\left(\GammaEq \Ate^\top\right)\\
    &= 1 - \frac{\nu}{d}\tr F_R(\nu) - a^\ast(1+\rho)^2 \nu \frac{\norm{F_R(\nu)\btr}^2}{d}\\
    &\simeq 1 - \nu \mathcal{M}_\tload(\nu),\label{eq:GA_ICL}
\end{align}    
where $\nu$ is the constant defined in \eqref{eq:nu_xi}. To reach the last step, we have used the estimate given in \eqref{eq:GEtr_asymp}.
 
Next, we use \eqref{eq:GammaBGamma} to characterize the first term on the right-hand side of \eqref{eq:eg_AB}. From the formulas in \eqref{eq:Gamma_e_precise} and \eqref{eq:AteBte_ICL}, we can check that
\begin{align}
    \frac{1}{d}\tr\left(\GammaEq B_\mathrm{test} (\GammaEq)^\top\right) &\simeq \left(\frac{1+\rho}{\alpha} + 1\right) \frac 1 d \tr\left(I - \nu F(\nu)\right)^2 \\&\simeq \left(\frac{1+\rho}{\alpha} + 1\right)\left(1 -2\nu\mathcal{M}_\tload(\nu) -\nu^2\mathcal{M}'_\tload(\nu)\right),\label{eq:GeBGe_ICL}
\end{align}
where the second step follows from \eqref{eq:GEtr_asymp} and \eqref{eq:GEtr2_asymp}.
From \eqref{eq:Etr_inv},
\begin{equation}\label{eq:BGE_ICL}
    \frac 1 d \tr(B_\mathrm{test}(\Etr + \xi I)^{-1}) \simeq \left(\frac{1+\rho}{\alpha} + 1\right) \frac 1 d \tr F_R(\nu) \simeq \left(\frac{1+\rho}{\alpha} + 1\right) \mathcal{M}_\tload(\nu).
\end{equation}
Similarly, we can check that
\begin{equation}\label{eq:BGE2_ICL}
    \frac 1 d \tr(B_\mathrm{test}(\Etr + \xi I)^{-2}) \simeq \left(\frac{1+\rho}{\alpha} + 1 \right) \frac 1 d \tr F^2_R(\nu) \simeq -\left(\frac{1+\rho}{\alpha} + 1 \right) \mathcal{M}'_\tload(\nu).
\end{equation}
Substituting \eqref{eq:GeBGe_ICL}, \eqref{eq:BGE_ICL}, and \eqref{eq:BGE2_ICL} into \eqref{eq:GammaBGamma} gives us
\begin{equation}\label{eq:GBG_ICL}
    \frac 1 d \tr(\Gamma^\ast B ( \Gamma^\ast)^\top) \simeq \left(\frac{1+\rho}{\alpha} + 1\right)\left(1 -2\nu\mathcal{M}_\tload(\nu) -\nu^2\mathcal{M}'_\tload(\nu) - c_e\left[\mathcal{M}_\tload(\nu) + \xi \mathcal{M}'_\tload(\nu)\right]\right),
\end{equation}
where $c_e$ is the constant given in \eqref{eq:GammaBGamma_c}. Combining \eqref{eq:GA_ICL}, \eqref{eq:GBG_ICL}, and \eqref{eq:eg_AB}, we are done.

In what follows, we further simplify the characterizations in Result~\ref{res:eg_ICL} by considering the ridgeless limit, \emph{i.e.}, when $\lambda \to 0^+$. 

\begin{result}\label{res:eg_ICL_ridgeless}
    Let
\begin{equation}\label{eq:mu_x_m}
    q^\ast \bydef \frac{1+\rho}{\alpha},\qquad m^\ast \bydef \mathcal{M}_\kappa\left({q^\ast}\right), \qquad \text{and} \qquad \mu^\ast \bydef q^\ast \mathcal{M}_{\tload/\load}(q^\ast),
\end{equation}
where $\mathcal{M}_\kappa(x)$ is the function defined in \eqref{eq:m_Wishart}. Then
\begin{equation}\label{eq:eg_ICL_ridgeless}
\begin{aligned}
    &\eiclrl \bydef \lim_{\lambda \to 0^+} \eicl(\load, \cload,\tload, \nv, \lambda)\\
    &= \begin{cases}
    {\frac{\tau(1+q^\ast)}{1-\tau}\left[1-\tau(1-\mu^\ast)^2+\mu^\ast(\rho/q^\ast-1)\right]} {-2\tau(1-\mu^\ast)+(1+\rho)} &\quad \tau < 1\\
    \left(q^\ast+1\right)\left(1 - 2q^\ast m^\ast -(q^\ast)^2 \mathcal{M}'_\tload(q^\ast) + \frac{(\rho + q^\ast  - (q^\ast)^2 m^\ast) m^\ast}{\tau-1}\right) -2 (1-q^\ast m^\ast) + (1+\nv) &\quad \tau > 1
    \end{cases},
\end{aligned}
\end{equation}
where $\mathcal{M}'_\tload(\cdot)$ denotes the derivative of $\mathcal{M}_\tload(x)$ with respect to $x$.
\end{result}

We start with the case of $\tau < 1$. Examining the self-consistent equation in \eqref{eq:xi_equation}, we can see that the parameter $\xi$ tends to a nonzero constant, denoted by $\xi^\ast$, as $\lambda \to 0^+$. It follows that the original equation in \eqref{eq:xi_equation} reduces to 
\begin{equation}\label{eq:xi_ridgeless_tau_l1}
{\xi^\ast}\mathcal{M}_\kappa\left(\frac{1+\rho}{\alpha} + \xi^\ast\right) = 1 - \tau.
\end{equation}
Introduce a change of variables
\begin{equation}\label{eq:mu_xi}
    \mu^\ast \bydef \frac{(1-\tau)(1+\rho)}{\alpha \tau \xi^\ast}.
\end{equation}
By combining \eqref{eq:xi_ridgeless_tau_l1} and the characterization in \eqref{eq:m_Wishart_identity}, 
we can directly solve for $\mu$ and get $\mu^\ast = q^\ast \mathcal{M}_{\tload/\load}(q^\ast)$ as given in \eqref{eq:mu_x_m}. The characterization in \eqref{eq:eg_ICL_ridgeless}  (for the case of $\tau < 1$) then directly follows from \eqref{eq:GA_ICL}, \eqref{eq:GBG_ICL}, and Result \ref{res:Ge_AB} after some lengthy calculations. 

Next, we consider the case of $\tau > 1$. It is straightforward to verify from \eqref{eq:xi_equation} that
    \begin{equation}
        \xi = \frac{\tau}{\tau-1}\lambda + \mathcal{O}(\lambda^2).
    \end{equation}
Thus, when $\tau > 1$, $\xi \to 0$ as $\lambda \to 0^+$. It follows that
\begin{equation}
    \lim_{\lambda \to 0^+} \nu = \lim_{\lambda \to 0^+} \left(\frac{1+\rho}{\alpha} + \xi\right) = q^\ast \quad \text{and}\quad \lim_{\lambda \to 0^+} \mathcal{M}_\tload(\nu) = m^\ast.
\end{equation}
Substituting these estimates into \eqref{eq:GA_ICL}, \eqref{eq:GBG_ICL}, and \eqref{res:Ge_AB}, we then reach the characterizations in \eqref{eq:eg_ICL_ridgeless} for the case of $\tau > 1$.

\subsection{The Generalization Error of In-Distribution Generalization}\label{sec:idg_generalisation_theory}

In what follows, we derive the asymptotic limit of the generalization error for the IDG task.

\begin{result}\label{res:eg_IDG}
Consider the test distribution $\Ptest$ associated with the IDG task. We have
\begin{equation}\label{eq:eg_IDG}
    e(\Gamma^\ast) \simeq \eidg(\load, \cload,\tload, \nv, \lambda) \bydef \tau \frac{\rho +\nu-\nu^2 \mathcal{M}_\tload(\nu) - \xi\left[1 - {2\nu}\mathcal{M}_\tload(\nu) - {\nu^2}\mathcal{M}'_\tload(\nu)\right]}{ \tau  - \left[1 - {2\xi}\mathcal{M}_\tload(\nu) - 
    \xi^2\mathcal{M}'_\tload(\nu)\right]},
\end{equation}
where $\xi$ the unique positive solution to the self-consistent equation \eqref{eq:xi_equation} and $\nu$ is the constant given in \eqref{eq:nu_xi}.
\end{result}

Similar to our derivation of Result~\ref{res:eg_ICL}, we only need to use \eqref{eq:Gamma_A_trace} and \eqref{eq:GammaBGamma} to characterize the asymptotic limits of the first and second terms on the right-hand side of \eqref{eq:eg_AB}. Note that, for the IDG task, $\Ate = \Atr$. It follows from \eqref{eq:Gamma_A_trace} and \eqref{eq:Gamma_e_Atr_asymp} that
\begin{equation}\label{eq:Gamma_A_IDG}
    \frac 1 d \tr(\Gamma^\ast \Ate^\top) \simeq 1 - \nu + \nu^2 \mathcal{M}_\tload(\nu).
\end{equation}
Similarly, since $\Bte = \Etr$, we can verify from \eqref{eq:Gamma_e} that
\begin{align}\label{eq:Gamma_e_B_Gamma_e_IDG}
    \frac{1}{d}\tr\left(\GammaEq \Bte (\GammaEq)^\top\right) &= \frac{1}{d} \tr(\GammaEq \Atr^\top) - \frac{\xi}{d} \tr(\GammaEq (\GammaEq)^\top)\\
    &\simeq 1 - \nu + \nu^2 \mathcal{M}_\tload(\nu) - \xi\left[1 - 2\nu \mathcal{M}_\tload(\nu) - \nu^2 \mathcal{M}'_\tload(\nu)\right],
\end{align}
where the second step follows from \eqref{eq:Gamma_e_Atr_asymp} and \eqref{eq:Gamma_e_Gamma_e_asymp}. 
Moreover,
\begin{equation}\label{eq:BE_BE2_IDG}
    \frac 1 d \tr\left(\Bte\left[(\Etr + \xi I)^{-1}- \xi(\Etr + \xi I)^{-2}\right]\right) = 1 - 2 \xi \mathcal{M}_\tload(\nu) - \xi^2 \mathcal{M}'_\tload(\nu).
\end{equation}
Substituting \eqref{eq:Gamma_e_B_Gamma_e_IDG} and \eqref{eq:BE_BE2_IDG} into \eqref{eq:GammaBGamma}, we have
\begin{equation}
\begin{aligned}
    &\frac 1 d \tr(\Gamma^\ast B ( \Gamma^\ast)^\top)\\
    &\ \simeq \tau \frac{\rho +\nu-\nu^2 \mathcal{M}_\tload(\nu) - \xi\left[1 - {2\nu}\mathcal{M}_\tload(\nu) - {\nu^2}\mathcal{M}'_\tload(\nu)\right]}{ \tau  - \left[1 - {2\xi}\mathcal{M}_\tload(\nu) - 
    \xi^2\mathcal{M}'_\tload(\nu)\right]}
    + 2(1 - \nu + \nu^2 \mathcal{M}_\tload(\nu)) - (1+\nv).
\end{aligned}
\end{equation}
The final result in \eqref{eq:eg_IDG} then follows from combining the above expression with \eqref{eq:Gamma_A_IDG} and \eqref{eq:eg_AB}.

Finally, we derive the ridgeless limit of the characterization given in Result~\ref{res:eg_IDG}.

\begin{result}
Let $q^\ast$, $m^\ast$, and $\mu^\ast$ be the scalars defined in \eqref{eq:mu_x_m}. We have
\begin{align}
    \eidgrl &\bydef \lim_{\lambda \to 0^+} \eidg(\load, \cload,\tload, \nv, \lambda) = \begin{cases}
    \frac{\tau}{1-\tau}\left(\frac{\rho + q^\ast - 2 q^\ast(1-\tau)({q^\ast}/{\xi^\ast}+1)}{1 - p^\ast(1-\tau)} + \frac{\tau \mu^\ast(q^\ast+\xi^\ast)^2}{q^\ast}\right) &\qquad \tau < 1\\
         \frac{\tau}{\tau-1}[\rho + q^\ast (1 - q^\ast m^\ast)] &\qquad \tau > 1
    \end{cases},\label{eq:eidg_ridgeless}
\end{align}
where $\xi^\ast = \frac{(1-\tau)q^\ast}{\tau \mu^\ast}$ and $p^\ast = \big(1 - {\kappa}\big(\frac{\kappa \xi^\ast}{1-\tau}+1\big)^{-2}\big)^{-1}$.
\end{result}

The derivation of this result closely follows that of Result~\ref{res:eg_ICL_ridgeless}. We analyze the cases of $\tau < 1$ and $\tau > 1$ separately. For $\tau < 1$, the equation in \eqref{eq:xi_equation} simplifies to \eqref{eq:xi_ridgeless_tau_l1} as $\lambda \to 0^+$. For $\tau > 1$, $\xi$ approaches zero as $\lambda \to 0^+$. Substituting these estimates into \eqref{eq:eg_IDG} then yields \eqref{eq:eidg_ridgeless} after some detailed calculations.

\subsection[A note on ordering of limits]{A note on ordering of $\lambda \to 0^{+}$ and $\alpha \to \infty$ limits}\label{sec:limit_commutations} To understand why the divergence in $\alpha$ occurs for certain configurations of $\kappa$ and $\tau$, we investigate the problem where $\alpha$ is taken to infinity before $\lambda \to 0^{+}$, which is equivalent to first taking $\ell$ to infinity. This means that our features $H_Z$ have no token $x_i$ or noise $\eta$ disorder remaining, and instead depend only on $w$ and query $x_{\ell+1}$ as \begin{equation}
    H_Z = \begin{bmatrix}
        x_{\ell+1}w^\top & \frac{1}{d}\left(\|w\|^2 + \rho\right)x_{\ell+1}
    \end{bmatrix}
\end{equation} Computing the ridgeless limit of the generalisation error for the finite $n$ predictor \eqref{eq:ridge_LT_SI} on the above $H_Z$ gives the following ridgeless, infinite-$\alpha$ generalisation error \begin{align}\label{eq:alphalambdalimit}
    \lim_{\lambda \to 0^+}&\lim_{\alpha\to\infty} \eicl(\load, \cload,\tload, \nv, \lambda) =
    \begin{dcases}
        1 - \kappa + \rho + \frac{\rho\kappa^2}{(\tau-\kappa)(1-\kappa)} &  \kappa \leq \text{min}(\tau,1) \\
        1 - \tau + \rho + \rho\frac{\kappa\tau}{(\kappa-\tau)(1-\tau)} & \tau < 1, \kappa > \tau 
        \\
        \rho + \frac{\rho\kappa}{(\tau-1)(\kappa-1)} & \tau > 1, \kappa> 1
    \end{dcases},
\end{align} noting that the second and third cases match their equivalent finite cases \begin{align}\label{eq:lambdaalphalimit_tsmall_SI}
    \lim_{\alpha \to \infty} \eiclrl &= 
    \begin{dcases}
        \infty & \kappa \leq \text{min}(\tau, 1), 
        \\
        1 - \tau + \rho + \frac{\rho \kappa \tau}{(\kappa-\tau)(1-\tau)} & \tau < 1, \kappa > \tau, \\
        \rho + \frac{\rho \kappa}{(\kappa - 1)(\tau-1)} & \tau > 1, \kappa > 1.
    \end{dcases}
\end{align} from the main manuscript. This result follows from applying the asymptotics given in prior works for the generalization error of ridgeless regression with structured covariates \cite{hastie2022surprises,dubova2023universality,atanasov2024scaling}. By comparison with \eqref{eq:lambdaalphalimit_tsmall_SI}, we see that switching the order of limits in $\alpha$ and $\lambda$ has mitigated the divergence for $\kappa \leq \min(\tau,1)$.

As these limits do not commute, which order is correct? Simulations with arbitrarily small numerical ridge parameter at large $\alpha$ track the  \eqref{eq:lambdaalphalimit_tsmall_SI} curve, not the \eqref{eq:alphalambdalimit} curve. These numerical experiments are computed with the ridge scaling $n/d$ that we use elsewhere. Were we to scale this regularisation with some power of $\alpha$, we could remove this large $\alpha$ divergence; whether or not that choice is either theoretically or practically motivated is a question we leave open for future work.

\section{Task Generalization and Bayesian Estimators}\label{sec:bayesestimatorerror}
In the main document, we have compared the performance of our linear transformer with that of two Bayesian estimators: the discrete task prior estimator and the optimal ridge estimator. Here we will derive various claims about these estimators made in the main discussion, particularly relating to their behavior against task diversity $\kappa$. 

\subsection{Ridge estimator error} Here we prove claim that the ridge estimator error is characterized by $$e^{\text{ridge}}_\text{ICL} = \rho\left(1+\frac{1}{\alpha}\mathcal{M}_{\alpha}\left(\frac{\rho}{\alpha}\right)\right) = e^{\text{ridge}}_\text{IDG}\,.$$ First, it is obvious that we must have $e^{\text{ridge}}_\text{ICL} = e^{\text{ridge}}_\text{IDG}$, as a ridge estimator by definition only operates on a single context (i.e. a single task); it cannot distinguish between $w$ drawn from $\mathcal{N}(0,I_d)$ or $w$ drawn uniformly from $\{w_1,\cdots,w_k\}$ each drawn from $\mathcal{N}(0,I_d)$. Thus in the following we will not distinguish between these two testing regimes, and refer only to task $w \sim \mathcal{N}(0,I_d)$.

Let $x_1, \ldots, x_\ell$ be a collection of data vectors. We observe
\begin{equation}
    y_i = x_i^\top w + \epsilon_i, \quad \text{for } i = 1, 2, \ldots, \ell,
\end{equation}
where $\epsilon_i \sim \mathcal{N}(0, \rho)$ is the observation noise. Let $y = [y_i ]_{i \le \ell} \in \mathbb{R}^\ell$. Then we can write down second moments as
\begin{equation}\label{eq:Bayesian_correlation}
    \E[w w^\top] = I_d, \qquad \E[y w^\top] = X^\top, \qquad \text{and } \E[y y^\top] = X^\top X + \rho I_\ell, 
\end{equation}
where $X = [x_1, x_2, \ldots, x_\ell]$ is the data matrix in $\R^{d \times \ell}$. Since $y$ and $w$ are joint Gaussian random vectors, we can write
\begin{equation}
    \beta = (X X^\top + \rho I_d)^{-1} X y + Z,
\end{equation}
where $Z$ is a centered Gaussian vector independent of $y$. Note that the matrix in front of $y$ in the above expression is chosen so that the cross-correlation $\E[y w^\top]$ matches that in \eqref{eq:Bayesian_correlation}. It then immediately follows that the Bayesian estimator of $w$ is
\begin{equation}
    w^\text{ridge}_\text{Bayes} = (X X^\top + \rho I_d)^{-1} X y.
\end{equation}
Given a new input vector $x_{\ell+1}$ and its label $y_{\ell+1} = x_{\ell+1}^\top w + \epsilon_{\ell+1}$, the Bayesian optimal estimator of the new label is
\begin{align}
    \hat y_{\ell+1} &= x_{\ell+1}^\top w^\text{ridge}_\text{Bayes},
\end{align}
and the corresponding Bayesian prediction error is 
\begin{align}
    e_\text{Bayesian} &= \E[(y_{\ell+1} - x_{\ell+1}^\top w^\text{ridge}_\text{Bayes})]\\
    &= \rho + \frac 1 d \E\left[\norm{w - w^\text{ridge}_\text{Bayes}}^2\right]\\
    &\simeq \rho \left(1 + \frac 1 \alpha \mathcal{M}_\alpha\left(\frac{\rho}{\alpha(1-\gamma)}\right)\right),
\end{align}
where $\mathcal{M}_\alpha(x)$ is the function defined in \eqref{eq:m_Wishart_identity}. In the large $\alpha$ limit, we have
\begin{equation}
  e_\text{Bayesian} \simeq \rho(1 + 1/\alpha).  
\end{equation}

\subsection{Asymptotics of the linear transformer in $\kappa$} 

We wish to expand \eqref{eq:eg_ICL_ridgeless} and \eqref{eq:eidg_ridgeless} in large $\kappa$ to study the decay of $g_\text{task} = e_\text{ICL}-e_\text{IDG}$ toward 0 as $\kappa \to \infty$. By expanding auxillary variables in $\kappa$, we have 
\begin{align}
    \mathcal{M}_\kappa(q) \approx \frac{1}{1+q}\left(1+\frac{1}{(1+q)^2\kappa}\right)\,, \qquad \mathcal{M}'_\kappa(q) \approx -\frac{1}{(1+q)^2}\left(1+\frac{3}{(1+q)^2\kappa}\right) 
\end{align}
and \begin{align}
    \mu = q\mathcal{M}_{\kappa/\tau}(q) \approx \frac{q}{1+q}\left(1 + \frac{\tau}{(1+q)^2\kappa}\right)\,,\qquad 
    \xi \approx \frac{1-\tau}{\tau}(1+q)\left(1 - \frac{\tau}{(1+q)^2\kappa}\right)\,,\qquad 
    p \approx 1+\frac{\tau^2}{(1+\tau)^2\kappa}\,.
\end{align} Substituting these simplifications into \eqref{eq:eg_ICL_ridgeless} and \eqref{eq:eidg_ridgeless} gives the desired characterization $$g^\text{LT}_\text{task} = 0 + \frac{1}{\kappa}\begin{cases}
    c_1 & \tau < 1 \\
    c_2 & \tau > 1 \\
\end{cases} + \mathcal{O}\left(\frac{1}{\kappa^2}\right)$$ for \begin{align}
    c_1 &\equiv \frac{2 \left(\left(-2 q-\frac{1}{2}\right) \tau^{2}+\left(q^{2}+\left(-\rho +3\right) q+1-\rho \right) \tau -\left(1+q\right) \left(q+\frac{1}{2}-\frac{\rho}{2}\right)\right) \tau}{\left(-1+\tau \right) \left(1+q\right)^{3}} \,,\\
    c_2 &\equiv 2\frac{q^2}{(1+q)^3} + \frac{1}{(\tau-1)(q+1)^2}\left(\rho+q-\frac{q^2}{q+1}\right) \,.
\end{align}

\subsection{Asymptotics of the dMMSE estimator in $\kappa$}\label{sec:asymptoticsofdmmse}
To study the slowness of the dMMSE estimator more explicitly, consider the $\alpha\to\infty$ limit. The exponential weight terms in the estimator in this limit behave as
$$e^{-\frac{1}{2\nv}\sum_{i=1}^{\cl}\left(y_i-w_{j}^{\top}x_i\right)^2} \to e^{-\frac{\ell}{2\rho}\left(\frac{1}{d}\|w^\ast-w_j\|^2 + \rho\right)}\,,$$ and these weightings exponentially favor choosing $\w_j$ that minimizes $\|w^\ast-w_j\|^2$ over the set of $k$ training tasks $w_j$. It's immediately clear that $e_\text{IDG}^\text{dMMSE} = \rho$ in this limit as the minimal value of $\|w^\ast-w_j\|^2$ when $w^\ast \in \{w_1,\cdots,w_k\}$ is 0. Taking $$w_\text{est}(w^\ast,w_i) = \text{argmin}_{i\in[k]} \|w^\ast-w_i\|^2$$ we have \begin{align}
    g^\text{dMMSE}_\text{task} 
    &= \frac1d \mathbb{E}_{w^\ast \sim \mathcal{P}_\text{test}} \left[ \mathbb{E}_{w_i\sim \mathcal{P}_\text{train}}\left[ \min_{i\in[k]} \|w^\ast - w_i\|^2\right]\right].
\end{align} It is equivalent in high-dimensions to work with $w^\ast,w_i$ sampled uniformly on the sphere $\mathcal{S}^{d-1}(\sqrt{d}).$ For analytical tractability, we will proceed with this. We exploit spherical symmetry in both $w^\ast$ and $w_i$ to simplify \begin{align} 
\|w^\ast - w_i\|^2 &= 2d-2\sqrt{d}x_i \qquad \text{ for $x_i$ a single component of vector $w_i \sim \text{Unif}(\mathcal{S}^{d-1}(\sqrt{d}))$}\\
&=  2d - 2\sqrt{d}(2\sqrt{d}\,b_i - \sqrt{d}) \qquad \text{for $b_i\sim\text{Beta}\left(\frac{d-1}{2},\frac{d-1}{2}\right)$} \\
&= 4d - 4db_i
\end{align}
This gives \begin{align}
    g_\text{task} &= \rho + 4 - 4\,\mathbb{E}_{b_i\sim\text{Beta}\left(\frac{d-1}{2},\frac{d-1}{2}\right)} \left[\max_{i\in[k]}b_i \right].
\end{align} We wish to study the behavior of this expected maximum of a set of $k$ Beta variables as $k$ grows large. We are thus interested in calculating
\begin{align}
\mathbb{E}_{b_i\sim\text{Beta}\left(\frac{d-1}{2},\frac{d-1}{2}\right)} \left[\max_{i\in[k]}b_i \right],
\end{align}
Let $D = \frac{d-1}{2}$. Following the standard procedure for deriving the PDF for the distribution of the extreme value, have 
\begin{align}
p(b_i = x) = \frac{1}{B(D)} (1-x)^{D-1}x^{D-1}
\end{align}
where
\begin{align}
B(D) = \frac{\Gamma(D)^2}{\Gamma(2D)}.
\end{align}
Then
\begin{align}
\mathbb{P}(b_i \leq M, \forall i \in [k]) = \left(\frac{1}{B(D)} \int_0^M dx (1-x)^{D-1}x^{D-1}\right)^k
\end{align}
and
\begin{align}
p(\max_{i\in[k]} b_i = M) = \frac{k}{B(D)} M^{D-1}(1-M)^{D-1} \left(\frac{1}{B(D)}\int_0^M dx \, (1-x)^{D-1}x^{D-1}\right)^{k-1}
\end{align}
and
\begin{align}
\mathbb{E}_{b_i\sim\text{Beta}\left(\frac{d-1}{2},\frac{d-1}{2}\right)} \left[\max_{i\in[k]}b_i \right] = \frac{k}{B(D)} \int_0^1 dM M^{D}(1-M)^{D-1} \left(\frac{1}{B(D)}\int_0^M dx\,  (1-x)^{D-1}x^{D-1}\right)^{k-1}
\end{align}

Note that the integral can be written as
\begin{align}
\frac{k}{B(D)}\int_{0}^1dM \,M^{D}(1-M)^{D-1} \exp\left((k-1)\log \frac 1{B(D)}\int_{0}^M dx\, (1-x)^{D-1}x^{D-1}\right)
\end{align}
To proceed, let's consider the variable substitution
\begin{align}\label{var}
e^{-t} = \frac 1{B(D)}\int_{0}^M dx\, (1-x)^{D-1}x^{D-1}.
\end{align}
By this definition, $t(M=0) = \infty $, $t(M=1) = 0$
Plugging in
\begin{align}
k \int_{0}^\infty dt \, M(t) e^{-kt}.
\end{align}
Naively, it looks like one could get a simple series expansion in $1/k$ through integration by parts, but this will fail because $\left.\frac{dM}{dt}\right |_{t=0} = \infty$ for $D>1$. One can see this by differentiating \eqref{var} with respect to $t$. Instead, have 
\begin{align}
k \int_{0}^\infty dt \, M(t) e^{-kt} = \int_{0}^\infty dt \, M(t/k) e^{-t} = 1 - \int_{0}^\infty dt \, \epsilon(t/k) e^{-t}
\end{align}
where we defined
\begin{align}
\epsilon(t/k) \equiv 1 - M(t/k).
\end{align}
Now, we can expand $\epsilon(t/k)$ and integrate. We have to find an appropriate series expansion from the implicit definition of $\epsilon(t/k)$.
\begin{align}
e^{-t/k} &= \frac{1}{B(D)}\int_{0}^{1-\epsilon(t/k)} dx\, (1-x)^{D-1}x^{D-1} 
\end{align}
%
To get the leading behavior, one proceeds by expanding in leading order to $1-M$ or the right and to $t/k$ on the left
\begin{align}
1 - \frac{t}k = 1 - \frac{\epsilon(t/k)^D}{B(D)} \quad \implies \epsilon \approx \left(\frac{B(D)t}k\right)^{1/D}
\end{align}
which in turn implies the leading behavior of the expectation we are looking for is
\begin{align}
1 - \frac{B(D)^{1/D}\Gamma(1/D+1)}{k^{1/D}}
\end{align}

From this calculation we recover that the decrease of $g_\text{task}^\text{dMMSE}$ towards 0 is explicitly dimensionally cursed at $\alpha \to \infty$. We expect a similar rate to hold for finite $\alpha$ by the curse-of-dimensionality arguments given previously.

\section{Equivalent Statistical Representations}

In this appendix, we present an equivalent (but simplified) statistical model for the regression vector $H_Z$. This statistically-equivalent model will simplify the moment calculations in \sref{moment} and the random matrix analysis in \sref{deterministic_equivalent}. 

\begin{lemma}\label{lemma:H_tilde_rep} Let $\tv$ be a given task vector with $\norm{\tv} = \sqrt{d}$. Meanwhile, let $a \sim \mathcal{N}(0, 1)$, $s \sim \mathcal{N}(0, 1)$, $\epsilon \sim \mathcal{N}(0, \rho)$ be three scalar normal random variables, and $q \sim \mathcal{N}(0, I_{\cl-1})$, $g \sim \mathcal{N}(0, I_{d-1})$, $u \sim \mathcal{N}(0, I_{d-1})$, and $v_\epsilon \sim \mathcal{N}(0, \rho I_\cl)$ be isotropic normal random vectors. Moreover, $\tv$ and all of the above random variables are mutually independent. We have the following equivalent statistical representation of the pair $(H_Z, y_{\cl+1})$:
    \begin{equation}\label{eq:H_rep}
        H_Z \overset{(d)}{=} (d/\cl) M_{\tv}\begin{bmatrix}
            s \\ u
        \end{bmatrix}\begin{bmatrix}
        h^\top M_\tv, & ( a/\sqrt{d} + \theta_\epsilon)^2 / \sqrt{d} + \theta_q^2 /\sqrt{d}
        \end{bmatrix},
    \end{equation}
    and
    \begin{equation}\label{eq:y_target_rep}
        y_{\cl+1} \overset{(d)}{=}  s + \epsilon.
    \end{equation}
In the above displays, $M_{\tv}$ denotes a \emph{symmetric} and \emph{orthonormal} matrix such that 
    \begin{equation}\label{eq:Mtv}
        (M_{\tv}) e_1 = \frac{\tv}{\norm{\tv}},
    \end{equation}
where $e_1$ denotes the first natural basis vector in $\R^d$; $h \in \R^d$ is a vector defined as 
\begin{equation}\label{eq:h_equiv}
        h \bydef \begin{bmatrix}
        \frac{{\theta_\epsilon} a}{\sqrt{d}} + \frac{ a^2}{d} +  \theta_q^2\\
        \big[(\theta_\epsilon +  a/\sqrt{d})^2 +  \theta_q^2\big]^{1/2} { g}/{\sqrt{d}}
        \end{bmatrix};
    \end{equation}
and $\theta_\epsilon$, $\theta_q$ are scalars such that 
    \begin{equation}\label{eq:theta_scalars}
        \theta_\epsilon = {\norm{v_\epsilon}}/{\sqrt d} \qquad\text{and}\qquad \theta_q = \norm{q}/\sqrt d.
    \end{equation}   
\end{lemma}
\begin{remark}
For two random variables $A$ and $B$, the notation $A \overset{(d)}{=} B$ indicates that $A$ and $B$ have identical probability distributions. Note that $A$ and $B$ can be either scalars [as in the case of \eqref{eq:y_target_rep}], or matrices of matching dimensions [as in the case of \eqref{eq:H_rep}].  
\end{remark}
\begin{remark}
    A concrete construction of the symmetric and orthonormal matrix $M_\tv$ satisfying \eqref{eq:Mtv} can be based on the Householder transformation \cite{householder1958unitary,lu2021dice,trefethen1997numerical}. 
\end{remark}
\begin{proof}
Recall that the data vector $x_{\cl+1}$ is independent of the task vector $\tv$. Then, by the rotational symmetry of the isotropic normal distribution, we can rewrite
\begin{equation}\label{eq:xn_rep}
    x_{\cl+1} \overset{(d)}{=} \frac{1}{\sqrt{d}} M_{\tv}\begin{bmatrix}
        s \\ u
    \end{bmatrix},
\end{equation}
where $s \sim \mathcal{N}(0, 1)$ and $u \sim \mathcal{N}(0, I_{d-1})$ are two independent normal random variables (vectors), and $M_\tv$ is the symmetric orthonormal matrix specified in \eqref{eq:Mtv}. Note that $y_{\cl+1} = x_{\cl+1}^\top \tv + \nrv$, with $\nrv \sim \mathcal{N}(0,\nv)$ denoting the noise. The representation in \eqref{eq:y_target_rep} then follows immediately from \eqref{eq:xn_rep} and the identity in \eqref{eq:Mtv}. 

To show \eqref{eq:H_rep}, we first reparameterize the $d \times \cl$ Gaussian data matrix $X$ as
\begin{equation}\label{eq:X_rep}
    X = M_{\tv} \begin{bmatrix}
        a & q^\top\\
        p & U
    \end{bmatrix} M_{v_\epsilon} / \sqrt d.
\end{equation}
In the above display, $a \sim \mathcal{N}(0, 1)$, $p \sim \mathcal{N}(0, I_{d-1})$, $q \sim \mathcal{N}(0, I_{\cl-1})$; $U \in \R^{(d-1)\times(\cl-1)}$ is a matrix with iid standard normal entries; and $M_{v_\epsilon}$ is a symmetric orthonormal matrix such that
\begin{equation}\label{eq:M_ve}
    M_{v_\epsilon} e_1 = \frac{v_\epsilon}{\norm{v_\epsilon}},
\end{equation}
where $e_1$ denotes the first natural basis vector in $\R^\cl$. Since the data matrix $X$, the task vector $\tv$, and the noise vector $v_\epsilon$ are mutually independent, it is straightforward to verify via the rotational symmetry of the isotropic normal distribution that both sides of \eqref{eq:X_rep}  have identical probability distributions. Using this new representation, we have
\begin{equation}
    X v_\epsilon = \theta_\epsilon M_{\tv} \begin{bmatrix}
        a\\
        p
    \end{bmatrix}.
\end{equation}
Meanwhile,
\begin{equation}\label{eq:H_rep1}
    X^\top \tv = M_{v_\epsilon}\begin{bmatrix}
        a\\
        q
    \end{bmatrix},
\end{equation}
and thus
\begin{equation}\label{eq:H_rep2}
    X X^\top \tv = \frac{1}{\sqrt{d}} M_{\tv}\begin{bmatrix}
        a^2 + \norm{q}^2\\
        a p + U q
    \end{bmatrix}.
\end{equation}
Combining \eqref{eq:H_rep1} and \eqref{eq:H_rep2} yields
\begin{align}
Xy &= XX^\top \tv + X v_\epsilon\\
    &= M_{\tv}\begin{bmatrix}
        \theta_\epsilon a +  a^2/\sqrt{d} + \theta_q^2 \sqrt{d}\\
        (\theta_\epsilon +  a/\sqrt{d}) p + U q/\sqrt d
    \end{bmatrix}.
\end{align}
Observe that $U q/\sqrt{d} \overset{(d)}{=} \theta_q\, p'$, where $p' \sim \mathcal{N}(0, I_{d-1})$ is a normal random variable independent of everything else. Using this reparametrization for $U q/\sqrt{d}$ and the fact that $p, p'$ are two independent Gaussian vectors, we can conclude that
\begin{equation}\label{eq:Xy_rep}
    \frac{1}{\sqrt{d}} Xy \overset{(d)}{=} M_\tv h,
\end{equation}
where $h$ is the random vector defined in \eqref{eq:h_equiv}.

Lastly, we consider the term $y^\top y$ in $H_Z$. Since $y = X^\top \tv + v_\epsilon$, 
\begin{align}
    y^\top y &= \norm{X^\top \tv + v_\epsilon}^2\\
    &= \norm{X^\top \tv + \theta_\epsilon \sqrt{d}M_{v_\epsilon}e_1}^2\\
    &= (a + \theta_\epsilon \sqrt{d})^2 + \theta_q^2 d,\label{eq:yy_rep}
\end{align}
where the second equality follows from \eqref{eq:M_ve} and to reach the last equality we have used the representation in \eqref{eq:H_rep1}. To show \eqref{eq:H_rep}, we recall the definition of $H_Z$ in \eqref{eq:H_Z_SI}. Substituting \eqref{eq:xn_rep}, \eqref{eq:Xy_rep} and \eqref{eq:yy_rep} into \eqref{eq:H_Z_SI}, we are done.
\end{proof}

\section{The Stieltjes Transforms of Wishart Ensembles}
\label{appendix:Wishart}

In this appendix, we first recall several standard results related to the Stieltjes transforms of Wishart ensembles. In our problem, we assume that there are $\ntv$ unique task vectors $\set{\tv_i}_{i \in [\ntv]}$ in the training set. Moreover, these task vectors $\set{\tv_i}_{i \in [\ntv]}$ are independently sampled from the uniform distribution on the sphere $\mathcal{S}^{d-1}(\sqrt{d})$ with radius $\sqrt{d}$. Let 
\begin{equation}\label{eq:F_R}
    F_R(\nu) \bydef (\Rtr + \nu I_d)^{-1},
\end{equation}
where $\Rtr$ is the sample covariance matrix of the task vectors as defined in \eqref{eq:btr_Rtr} and $\nu$ is a positive scalar.

Note that the distribution of $\Rtr$ is asymptotically equivalent to that of a Wishart ensemble. By standard random matrix results on the Stieltjes transforms of Wishart ensembles (see, \emph{e.g.}, \cite{bai2010spectral}), we have
\begin{equation}\label{eq:F_trace_asymp}
    \frac 1 d \tr F_R(\nu) \simeq \mathcal{M}_\tload(\nu)
\end{equation}
as $d, \ntv \to \infty$ with $\ntv / d = \tload$. Here,
\begin{equation}\label{eq:m_Wishart}
    \mathcal{M}_\tload(\nu) \bydef \frac{2}{\nu + 1 - 1/\tload + \left[(\nu+1-1/\tload)^2 + 4\nu/\tload\right]^{1/2}}.
\end{equation}
is the solution to the self-consistent equation
\begin{equation}\label{eq:m_Wishart_identity}
    \frac{1}{\mathcal{M}_\tload(\nu)} = \frac{1}{1+ \mathcal{M}_\tload(\nu)/\tload} + \nu.
\end{equation}
Moreover, 
\begin{equation}
    \frac 1 d \tr F^2(\nu) \simeq -  \mathcal{M}'_\kappa(\nu) =  \frac{\mathcal{M}^2_\kappa(\nu)}{1 - \frac{\kappa \mathcal{M}_\kappa^2(\nu)}{[\kappa+\mathcal{M}_\kappa(\nu)]^2}}.
\end{equation}

For the remainder of this appendix, we will further explore the self-consistent equation given by \eqref{eq:chi_Etr}. We will show that the solution $\chi_\ps$ and its derivative $\frac{d}{d \ps} \chi_\ps $, at $\ps = 0$, can be characterized by the function $\mathcal{M}_\tload(\nu)$ in \eqref{eq:m_Wishart}. To start, note that at $\ps = 0$, the equation in \eqref{eq:chi_Etr} can be written as
\begin{equation}\label{eq:chi_xi_0}
    \frac{\tau \chi_0}{1+\chi_0} = (1+1/d) - \frac{\lambda(1+\chi_0)}{d} \tr(\Etr + \lambda(1+\chi_0) I)^{-1}.
\end{equation}
Recall the definition of $\Etr$ given in \eqref{eq:Etr}. It is straightforward to verify that
\begin{equation}\label{eq:Etr_inv}
    ( \Etr  + \lambda(1+\chi_0) I_{d+1})^{-1} = \begin{bmatrix}
        F_R(\nu_0) + a^\ast (1+\rho)^2 F_R(\nu_0) \btr \btr^\top F_R(\nu_0) & -a^\ast (1+\rho) F_R(\nu_0) \btr\\
        -a^\ast (1+\rho) \btr^\top F_R(\nu_0) & a^\ast
    \end{bmatrix},
\end{equation}
where $F_R(\cdot)$ is the function defined in \eqref{eq:F_R}, 
\begin{equation}\label{eq:nu0}
    \nu_0 = \frac{1+\rho}{\alpha} + \lambda(1+\chi_0)
\end{equation}
and
\begin{equation}\label{eq:a_inv}
    \frac{1}{a^\ast} = (1 + \rho)^2 + \lambda(1+\chi_0) - (1+\rho)^2 \btr^\top F_R(\nu_0) \btr.
\end{equation}
From \eqref{eq:Etr_inv}, the equation \eqref{eq:chi_xi_0} becomes
\begin{equation}\label{eq:chi_xi_1}
    \frac{\tau \chi_0}{1+\chi_0} = (1+1/d) - \frac{\lambda(1+\chi_0)}{d} \tr F_R(\nu_0) - (1+\nv)^2 \frac{a^\ast \lambda(1+\chi_0)}{d} \norm{F_R(\nu_0) \btr}^2.
\end{equation}
By the construction of $F_R(\nu_0)$ and $\btr$, we can verify that
\begin{equation}\label{eq:Fbtr_norm}
    \btr^\top F_R(\nu_0) \btr \le 1 \qquad\text{and}\qquad \norm{F_R(\nu_0) \btr}^2 \le \frac{1}{\nu_0} \le \frac{\alpha}{1+\nv}.
\end{equation}
Substituting the first inequality above into \eqref{eq:a_inv} gives us
\begin{equation}
    a^\ast \lambda(1+\chi_0) \le 1.
\end{equation}
Combining this estimate with the second inequality in \eqref{eq:Fbtr_norm}, we can conclude that the last term on the right-hand side of \eqref{eq:chi_xi_1} is negligible as $d \to \infty$. Moreover, using the asymptotic characterization given in \eqref{eq:F_trace_asymp}, the equation \eqref{eq:chi_xi_1} leads to
\begin{equation}\label{eq:chi_xi_2}
    \frac{\tau \chi_0}{1+\chi_0} \simeq 1 - \lambda(1+\chi_0)\mathcal{M}_\tload(\nu_0).
\end{equation}
Introducing a change of variables
\begin{equation}
    \xi_0 = \lambda (1+\chi_0),
\end{equation}
and also recalling the definition of $\nu_0$ in \eqref{eq:nu0}, we can further transform \eqref{eq:chi_xi_2} to
\begin{equation}
\xi_0\mathcal{M}_\tload\left(\frac{1+\nv}{\alpha}+\xi_0\right) - \frac{\tau \lambda}{\xi_0} \simeq 1 - \tau.    
\end{equation}
Observe that the above is identical to the equation in \eqref{eq:xi_equation}, except for a small error term captured by $\simeq$. By the stability of \eqref{eq:xi_equation}, we can then conclude that
\begin{equation}\label{eq:xi0_xi_1}
 \xi_0 \simeq \xi,   
\end{equation}
thus verifying \eqref{eq:xi0_xi}.

Next, we compute $\chi'_0$, the derivative of $\chi_\ps$ (with respect to $\ps$) evaluated at $\ps = 0$. Differentiating \eqref{eq:chi_Etr} give us
\begin{equation}\label{eq:dchi_0}
    \tau \chi'_0 = \frac 1 d \tr\left[(\Etr +\xi_0 I)^{-1}\Big(\chi'_0 \Etr - \frac{(1+\chi_0)^2}{\tau}\Bte\Big)(\Etr +\xi_0 I)^{-1}\Etr\right].
\end{equation}
Thus,
\begin{align}\label{eq:dchi_LT}
    \frac{\tau \chi'_0}{(1+\chi_0)^2} &\simeq  \frac{\frac 1 d \tr\left(\Bte[(\Etr + \xi I)^{-1}- \xi(\Etr + \xi I)^{-2}]\right)}{1  - 2\xi \tr (\Etr + \xi I)^{-1}/d + \xi^2 \tr(\Etr + \xi I)^{-2}/d - \tau},
\end{align}
where we have used \eqref{eq:xi0_xi_1} to replace $\xi_0$ in \eqref{eq:dchi_0} by $\xi$, with the latter being the solution to the self-consistent equation in \eqref{eq:xi_equation}. Using the decomposition in \eqref{eq:Etr_inv} and following similar arguments that allowed us to simplify \eqref{eq:chi_xi_1} to \eqref{eq:chi_xi_2}, we can check that
\begin{equation}\label{eq:GEtr_asymp}
    \frac{1}{d} \tr (\Etr + \xi I)^{-1} \simeq \frac{1}{d} \tr S(\Etr + \xi I)^{-1}S^\top \simeq \frac 1 d \tr F\left(\frac{1+\nv}{\cload} + \xi\right) \simeq \mathcal{M}_\tload\left(\frac{1+\nv}{\cload} + \xi\right),
\end{equation}
and
\begin{equation}\label{eq:GEtr2_asymp}
    \frac{1}{d} \tr (\Etr + \xi I)^{-2} \simeq \frac{1}{d} \tr S(\Etr + \xi I)^{-2}S^\top \simeq \frac 1 d \tr F^2\left(\frac{1+\nv}{\cload} + \xi\right) \simeq -\mathcal{M}'_\tload\left(\frac{1+\nv}{\cload} + \xi\right),
\end{equation}
where $S$ is a $d \times (d+1)$ matrix obtained by removing the last row of $I_{d+1}$, and $\mathcal{M}_\kappa(\cdot)$ is the function defined in \eqref{eq:m_Wishart}. Substituting \eqref{eq:GEtr_asymp} and \eqref{eq:GEtr2_asymp} into \eqref{eq:dchi_LT} yields
\begin{align}\label{eq:dchi}
    \frac{\tau \chi'_0}{(1+\chi_0)^2} &\simeq  \frac{\frac 1 d \tr\left(\Bte[(\Etr + \xi I)^{-1}- \xi(\Etr + \xi I)^{-2}]\right)}{1  - 2\xi \mathcal{M}_\tload\left(\frac{1+\nv}{\cload} + \xi\right) - \xi^2 \mathcal{M}'_\tload\left(\frac{1+\nv}{\cload} + \xi\right) - \tau}.
\end{align}

\end{document}